\documentclass[english]{article}
\usepackage[T1]{fontenc}
\usepackage[latin9]{inputenc}
\usepackage{amsmath}
\usepackage{amsthm}
\usepackage{amssymb}
\usepackage{graphicx}
\usepackage{esint}
\usepackage[numbers]{natbib}

\makeatletter
\theoremstyle{plain}
\newtheorem{thm}{\protect\theoremname}
\ifx\proof\undefined
\newenvironment{proof}[1][\protect\proofname]{\par
\normalfont\topsep6\p@\@plus6\p@\relax
\trivlist
\itemindent\parindent
\item[\hskip\labelsep\scshape #1]\ignorespaces
}{%
\endtrivlist\@endpefalse
}
\providecommand{\proofname}{Proof}
\fi
\theoremstyle{plain}
\newtheorem{cor}[thm]{\protect\corollaryname}
\theoremstyle{plain}
\newtheorem{lem}[thm]{\protect\lemmaname}
\theoremstyle{remark}
\newtheorem{rem}[thm]{\protect\remarkname}

\usepackage{nips15submit_e}\usepackage{times}

\nipsfinalcopy 

\usepackage{enumitem,amsthm}
\setenumerate{leftmargin=.25in} 

\usepackage{psfrag}

\usepackage{endnotes}
\usepackage{hyperref}

\makeatother

\usepackage{babel}
\providecommand{\corollaryname}{Corollary}
\providecommand{\lemmaname}{Lemma}
\providecommand{\remarkname}{Remark}
\providecommand{\theoremname}{Theorem}

\begin{document}

\title{Maximum Likelihood Learning With Arbitrary Treewidth via Fast-Mixing
Parameter Sets}

\author{Justin Domke\\
NICTA, Australian National University\\
\texttt{justin.domke@nicta.com.au}}
\maketitle
\begin{abstract}
Inference is typically intractable in high-treewidth undirected graphical
models, making maximum likelihood learning a challenge. One way to
overcome this is to restrict parameters to a tractable set, most typically
the set of tree-structured parameters. This paper explores an alternative
notion of a tractable set, namely a set of ``fast-mixing parameters''
where Markov chain Monte Carlo (MCMC) inference can be guaranteed
to quickly converge to the stationary distribution. While it is common
in practice to approximate the likelihood gradient using samples obtained
from MCMC, such procedures lack theoretical guarantees. This paper
proves that for any exponential family with bounded sufficient statistics,
(not just graphical models) when parameters are constrained to a fast-mixing
set, gradient descent with gradients approximated by sampling will
approximate the maximum likelihood solution inside the set with high-probability.
When unregularized, to find a solution $\epsilon$-accurate in log-likelihood
requires a total amount of effort cubic in $1/\epsilon$, disregarding
logarithmic factors. When ridge-regularized, strong convexity allows
a solution $\epsilon$-accurate in parameter distance with effort
quadratic in $1/\epsilon$. Both of these provide of a fully-polynomial
time randomized approximation scheme.
\end{abstract}

\section{Introduction}

In undirected graphical models, maximum likelihood learning is intractable
in general. For example, Jerrum and Sinclair \citep{Jerrum1993Polynomialtimeapproximation}
show that evaluation of the partition function (which can easily be
computed from the likelihood) for an Ising model is \#P-complete,
and that even the existence of a fully-polynomial time randomized
approximation scheme (FPRAS) for the partition function would imply
that RP = NP.

If the model is well-specified (meaning that the target distribution
falls in the assumed family) then there exist several methods that
can efficiently recover correct parameters, among them the pseudolikelihood
\citep{Besag1975StatisticalAnalysisNon}, score matching \citep{Hyvarinen2005EstimationNonNormalized,Marlin2011AsymptoticEfficiencyDeterministic},
composite likelihoods \citep{Lindsay1988CompositeLikelihoodMethods,Varin2011OverviewCompositeLikelihood},
Mizrahi et al.'s \citep{Mizrahi2014LinearandParallel} method
based on parallel learning in local clusters of nodes and Abbeel et
al.'s \citep{Abbeel2006LearningFactorGraphs} method based on
matching local probabilities. While often useful, these methods have
some drawbacks. First, these methods typically have inferior sample
complexity to the likelihood. Second, these all assume a well-specified
model. If the target distribution is not in the assumed class, the
maximum-likelihood solution will converge to the M-projection (minimum
of the KL-divergence), but these estimators do not have similar guarantees.
Third, even when these methods succeed, they typically yield a distribution
in which inference is still intractable, and so it may be infeasible
to actually make use of the learned distribution.

Given these issues, a natural approach is to restrict the graphical
model parameters to a tractable set $\Theta$, in which learning and
inference can be performed efficiently. The gradient of the likelihood
is determined by the marginal distributions, whose difficulty is typically
determined by the treewidth of the graph. Thus, probably the most
natural tractable family is the set of tree-structured distributions,
where $\Theta=\{\theta:\exists\text{tree}~T,\forall(i,j)\not\in T,\theta_{ij}=0\}.$
The Chow-Liu algorithm \citep{Chow1968Approximatingdiscreteprobability}
provides an efficient method for finding the maximum likelihood parameter
vector $\theta$ in this set, by computing the mutual information
of all empirical pairwise marginals, and finding the maximum spanning
tree. Similarly, Heinemann and Globerson \citep{Heinemann2014InferningwithHigh}
give a method to efficiently learn high-girth models where correlation
decay limits the error of approximate inference, though this will
not converge to the M-projection when the model is mis-specified.

This paper considers a fundamentally different notion of tractability,
namely a guarantee that Markov chain Monte Carlo (MCMC) sampling will
quickly converge to the stationary distribution. Our fundamental result
is that if $\Theta$ is such a set, and one can project onto $\Theta$,
then there exists a FPRAS for the maximum likelihood solution inside
$\Theta$. While inspired by graphical models, this result works entirely
in the exponential family framework, and applies generally to any
exponential family with bounded sufficient statistics.

The existence of a FPRAS is established by analyzing a common existing
strategy for maximum likelihood learning of exponential families,
namely gradient descent where MCMC is used to generate samples and
approximate the gradient. It is natural to conjecture that, if the
Markov chain is fast mixing, is run long enough, and enough gradient
descent iterations are used, this will converge to nearly the optimum
of the likelihood inside $\Theta$, with high probability. This paper
shows that this is indeed the case. A separate analysis is used for
the ridge-regularized case (using strong convexity) and the unregularized
case (which is merely convex).

\section{Setup\label{sec:Setup}}

Though notation is introduced when first used, the most important
symbols are given here for more reference.
\begin{itemize}
\item $\theta$ - parameter vector to be learned
\item $\mathbb{M}_{\theta}$ - Markov chain operator corresponding to $\theta$
\item $\theta_{k}$ - estimated parameter vector at $k$-th gradient descent
iteration
\item $q_{k}=\mathbb{M}_{\theta_{k-1}}^{v}r$ - approximate distribution
sampled from at iteration $k$. ($v$ iterations of the Markov chain
corresponding to $\theta_{k-1}$ from arbitrary starting distribution
$r$.)
\item $\Theta$ - constraint set for $\theta$
\item $f$ - negative log-likelihood on training data
\item $L$ - Lipschitz constant for the gradient of $f$.
\item $\theta^{*}=\arg\min_{\theta\in\Theta}f(\theta)$ - minimizer of likelihood
inside of $\Theta$
\item $K$ - total number of gradient descent steps
\item $M$ - total number of samples drawn via MCMC
\item $N$ - length of vector $x$.
\item $v$ - number of Markov chain transitions applied for each sample
\item $C,\alpha$ - parameters determining the mixing rate of the Markov
chain. (Equation \ref{eq:markov-chain-speed})
\item $R_{a}$ - sufficient statistics norm bound.
\item $\epsilon_{f}$ - desired optimization accuracy for $f$
\item $\epsilon_{\theta}$ - desired optimization accuracy for $\theta$
\item $\delta$ - permitted probability of failure to achieve a given approximation
accuracy
\end{itemize}
This paper is concerned with an exponential family of the form
\[
p_{\theta}(x)=\exp(\theta\cdot t(x)-A(\theta)),
\]
where $t(x)$ is a vector of sufficient statistics, and the log-partition
function $A(\theta)$ ensures normalization. An undirected model can
be seen as an exponential family where $t$ consists of indicator
functions for each possible configuration of each clique \citep{Wainwright2008GraphicalModelsExponential}.
While such graphical models motivate this work, the results are most
naturally stated in terms of an exponential family and apply more
generally.

\begin{figure}
\hspace{-20pt}%
\parbox[t]{0.6\columnwidth}{%
\begin{itemize}
\item Initialize $\theta_{0}=0$.
\item For $k=1,2,...,K$

\begin{itemize}
\item Draw samples. For $i=1,...,M$, sample\vspace{-2bp}
\[
x_{i}^{k-1}\sim q_{k-1}:=\mathbb{M}_{\theta_{k-1}}^{v}r.
\]
\vspace{-14bp}

\item Estimate the gradient as\vspace{-5bp}
\[
f'(\theta_{k-1})+e_{k}\leftarrow\frac{1}{M}\sum_{i=1}^{M}t(x_{i}^{k-1})-\bar{t}+\lambda\theta.
\]
\vspace{-9bp}

\item Update the parameter vector as\vspace{-4bp}
\[
\theta_{k}\leftarrow\Pi_{\Theta}\left[\theta_{k-1}-\frac{1}{L}\left(f'(\theta_{k-1})+e_{k})\right)\right].
\]
\vspace{-8bp}

\end{itemize}
\item Output $\theta_{K}$ or $\frac{1}{K}\sum_{k=1}^{K}\theta_{k}$.\end{itemize}
}%
\parbox[t][1\totalheight][c]{0.44\columnwidth}{%
\begin{center}
\vspace{10bp}
\includegraphics[width=0.44\columnwidth]{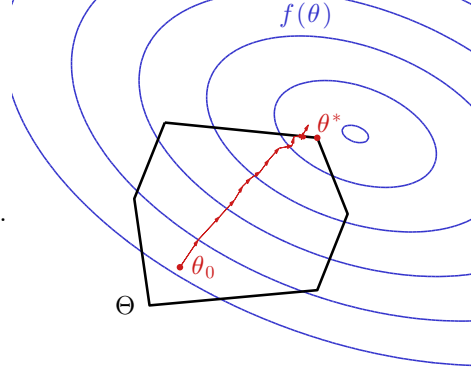}
\par\end{center}%
}

\caption{Left: \textbf{Algorithm 1}, approximate gradient descent with gradients
approximated via MCMC, analyzed in this paper. Right: A cartoon of
the desired performance, stochastically finding a solution near $\theta^{*}$,
the minimum of the regularized negative log-likelihood $f(\theta)$
in the set $\Theta$.\label{fig:outline-and-cartoon}}
\end{figure}

We are interested in performing maximum-likelihood learning, i.e.
minimizing, for a dataset $z_{1},...,z_{D}$,
\begin{equation}
f(\theta)=-\frac{1}{D}\sum_{i=1}^{D}\log p_{\theta}(z_{i})+\frac{\lambda}{2}\Vert\theta\Vert_{2}^{2}=A(\theta)-\theta\cdot\bar{t}+\frac{\lambda}{2}\Vert\theta\Vert_{2}^{2},\label{eq:likelihood}
\end{equation}
where we define $\bar{t}=\frac{1}{D}\sum_{i=1}^{D}t(z_{i}).$ It is
easy to see that the gradient of $f$ takes the form
\[
f'(\theta)=\mathbb{E}_{p_{\theta}}[t(X)]-\bar{t}+\lambda\theta.
\]

If one would like to optimize $f$ using a gradient-based method,
computing the expectation of $t(X)$ with respect to $p_{\theta}$
can present a computational challenge. With discrete graphical models,
the expected value of $t$ is determined by the marginal distributions
of each factor in the graph. Typically, the computational difficulty
of computing these marginal distributions is determined by the treewidth
of the graph-- if the graph is a tree, (or close to a tree) the marginals
can be computed by the junction-tree algorithm \citep{Koller2009ProbabilisticGraphicalModels:}.
One option, with high treewidth, is to approximate the marginals with
a variational method. This can be seen as exactly optimizing a ``surrogate
likelihood'' approximation of Eq. \ref{eq:likelihood} \citep{Wainwright2006EstimatingWrongGraphical}.

Another common approach is to use Markov chain Monte Carlo (MCMC)
to compute a sample $\{x_{i}\}_{i=1}^{M}$ from a distribution close
to $p_{\theta}$, and then approximate $\mathbb{E}_{p_{\theta}}[t(X)]$
by $(1/M)\sum_{i=1}^{M}t(x_{i})$. This strategy is widely used, varying
in the model type, the sampling algorithm, how samples are initialized,
the details of optimization, and so on \citep{Geyer1991MarkovchainMonte,Salakhutdinov2009LearninginMarkov,Schmidt2010GenerativePerspectiveMRFs,Papandreou2011PerturbandMAP,XavierDescombes1996EstimationMarkovRandom,Zhu1998FiltersRandomFields,Gu2001Maximumlikelihoodestimation,Asuncion2010LearningwithBlocks,Tieleman2008TrainingrestrictedBoltzmann,nan2005ContrastiveDivergenceLearning}.
Recently, Steinhardt and Liang \citep{Steinhardt2015LearningFastMixing}
proposed learning in terms of the stationary distribution obtained
from a chain with a nonzero restart probability, which is fast-mixing
by design.

While popular, such strategies generally lack theoretical guarantees.
If one were able to exactly sample from $p_{\theta}$, this could
be understood simply as stochastic gradient descent. But, with MCMC,
one can only sample from a distribution approximating $p_{\theta}$,
meaning the gradient estimate is not only noisy, but also biased.
In general, one can ask how should the step size, number of iterations,
number of samples, and number of Markov chain transitions be set to
achieve a convergence level.

The gradient descent strategy analyzed in this paper, in which one
updates a parameter vector $\theta_{k}$ using approximate gradients
is outlined and shown as a cartoon in Figure \ref{fig:outline-and-cartoon}.
Here, and in the rest of the paper, we use $p_{k}$ as a shorthand
for $p_{\theta_{k}}$, and we let $e_{k}$ denote the difference between
the estimated gradient and the true gradient $f'(\theta_{k-1})$.
The projection operator is defined by $\Pi_{\Theta}[\phi]=\arg\min_{\theta\in\Theta}||\theta-\phi||_{2}$.

We assume that the parameter set $\theta$ is constrained to a set
$\Theta$ such that MCMC is guaranteed to mix at a certain rate (Section
\ref{sub:Mixing-times-and}). With convexity, this assumption can
bound the mean and variance of the errors at each iteration, leading
to a bound on the sum of errors. With strong convexity, the error
of the gradient at each iteration is bounded with high probability.
Then, using results due to \citep{Schmidt2011ConvergenceRatesInexact}
for projected gradient descent with errors in the gradient, we show
a schedule the number of iterations $K$, the number of samples $M$,
and the number of Markov transitions $v$ such that with high probability,
\[
f\left(\frac{1}{K}\sum_{k=1}^{K}\theta_{k}\right)-f\left(\theta^{*}\right)\leq\epsilon_{f}\text{ \,\,or\,\, }\Vert\theta_{K}-\theta^{*}\Vert_{2}\leq\epsilon_{\theta},
\]
for the convex or strongly convex cases, respectively, where $\theta^{*}\in\arg\min_{\theta\in\Theta}f(\theta)$.
The total number of Markov transitions applied through the entire
algorithm, $KMv$ grows as $(1/\epsilon_{f})^{3}\log(1/\epsilon_{f})$
for the convex case, $(1/\epsilon_{\theta}^{2})\log(1/\epsilon_{\theta}^{2})$
for the strongly convex case, and polynomially in all other parameters
of the problem.

\section{Background}

\subsection{Mixing times and Fast-Mixing Parameter Sets\label{sub:Mixing-times-and}}

This Section discusses some background on mixing times for MCMC. Typically,
mixing times are defined in terms of the \textbf{total-variation distance}
$\Vert p-q\Vert_{TV}=\max_{A}\vert p(A)-q(A)\vert,$ where the maximum
ranges over the sample space. For discrete distributions, this can
be shown to be equivalent to $\Vert p-q\Vert_{TV}=\frac{1}{2}\sum_{x}\vert p(x)-q(x)\vert.$

We assume that a sampling algorithm is known, a single iteration of
which can be thought of an operator $\mathbb{M}_{\theta}$ that transforms
some starting distribution into another. The stationary distribution
is $p_{\theta}$, i.e. $\lim_{v\rightarrow\infty}\mathbb{M}_{\theta}^{v}q=p_{\theta}$
for all $q$. Informally, a Markov chain will be fast mixing if the
total variation distance between the starting distribution and the
stationary distribution decays rapidly in the length of the chain.
This paper assumes that a convex set $\Theta$ and constants $C$
and $\alpha$ are known such that for all $\theta\in\Theta$ and all
distributions $q$,\vspace{-10bp}

\begin{equation}
\Vert\mathbb{M}_{\theta}^{v}q-p_{\theta}\Vert_{TV}\leq C\alpha^{v}.\label{eq:mixing-assumption}
\end{equation}
This means that the distance between an arbitrary starting distribution
$q$ and the stationary distribution $p_{\theta}$ decays geometrically
in terms of the number of Markov iterations $v$. This assumption
is justified by the Convergence Theorem \citep[Theorem 4.9]{Levin2006Markovchainsand},
which states that if $\mathbb{M}$ is irreducible and aperiodic with
stationary distribution $p$, then there exists constants $\alpha\in(0,1)$
and $C>0$ such that
\begin{equation}
d(v):=\sup_{q}\Vert\mathbb{M}^{v}q-p\Vert_{TV}\leq C\alpha^{v}.\label{eq:markov-chain-speed}
\end{equation}
\vspace{-10bp}

Many results on mixing times in the literature, however, are stated
in a less direct form. Given a constant $\epsilon$, the \textbf{mixing
time} is defined by $\tau(\epsilon)=\min\{v:d(v)\leq\epsilon\}.$
It often happens that bounds on mixing times are stated as something
like $\tau(\epsilon)\leq\left\lceil a+b\ln\frac{1}{\epsilon}\right\rceil $
for some constants $a$ and $b$. It follows from this that $\Vert\mathbb{M}^{v}q-p\Vert_{TV}\leq C\alpha^{v}$
with $C=\exp(a/b)$ and $\alpha=\exp(-1/b).$

A simple example of a fast-mixing exponential family is the Ising
model, defined for $x\in\{-1,+1\}^{N}$ as\vspace{-15bp}

\[
p(x|\theta)=\exp\left(\sum_{(i,j)\in\text{Pairs}}\theta_{ij}x_{i}x_{j}+\sum_{i}\theta_{i}x_{i}-A(\theta)\right).
\]
A simple result for this model is that, if the maximum degree of any
node is $\Delta$ and $\vert\theta_{ij}\vert\leq\beta$ for all $(i,j)$,
then for univariate Gibbs sampling with random updates, $\tau(\epsilon)\leq\lceil\frac{N\log(N/\epsilon)}{1-\Delta\tanh(\beta)}\rceil$
\citep{Levin2006Markovchainsand}. The algorithm discussed in this
paper needs the ability to project some parameter vector $\phi$ onto
$\Theta$ to find $\arg\min_{\theta\in\Theta}||\theta-\phi||_{2}.$
Projecting a set of arbitrary parameters onto this set of fast-mixing
parameters is trivial-- simply set $\theta_{ij}=\beta$ for $\theta_{ij}>\beta$
and $\theta_{ij}\leftarrow-\beta$ for $\theta_{ij}<-\beta$.

For more dense graphs, it is known \citep{Hayes2006simpleconditionimplying,Dyer2009Matrixnormsand}
that, for a matrix norm $\Vert\cdot\Vert$ that is the spectral norm
$\Vert\cdot\Vert_{2}$, or induced $1$ or infinity norms, 
\begin{equation}
\tau(\epsilon)\leq\left\lceil \frac{N\log(N/\epsilon)}{1-\Vert R(\theta)\Vert}\right\rceil \label{eq:mixing-time}
\end{equation}
where $R_{ij}(\theta)=|\theta_{ij}|.$ Domke and Liu \citep{Domke2013ProjectingIsingModel}
show how to perform this projection for the Ising model when $\Vert\cdot\Vert$
is the spectral norm $\Vert\cdot\Vert_{2}$ with a convex optimization
utilizing the singular value decomposition in each iteration.

Loosely speaking, the above result shows that univariate Gibbs sampling
on the Ising model is fast-mixing, as long as the interaction strengths
are not too strong. Conversely, Jerrum and Sinclair \citep{Jerrum1993Polynomialtimeapproximation}
exhibited an alternative Markov chain for the Ising model that is
rapidly mixing for \emph{arbitrary} interaction strengths, provided
the model is ferromagnetic, i.e. that all interaction strengths are
positive with $\theta_{ij}\geq0$ and that the field is unidirectional.
This Markov chain is based on sampling in different ``subgraphs world''
state-space. Nevertheless, it can be used to estimate derivatives
of the Ising model log-partition function with respect to parameters,
which allows estimation of the gradient of the log-likelihood. Huber
\citep{Huber2012SimulationReductionsIsing} provided a simulation
reduction to obtain an Ising model sample from a subgraphs world sample.

More generally, Liu and Domke \citep{Liu2014ProjectingMarkovRandom}
consider a pairwise Markov random field, defined as
\[
p(x|\theta)=\exp\left(\sum_{i,j}\theta_{ij}(x_{i},x_{j})+\sum_{i}\theta_{i}(x_{i})-A(\theta)\right),
\]
and show that, if one defines $R_{ij}(\theta)=\max_{a,b,c}\frac{1}{2}\vert\theta_{ij}(a,b)-\theta_{ij}(a,c)\vert,$
then again Equation \ref{eq:mixing-time} holds. An algorithm for
projecting onto the set $\Theta=\{\theta:\Vert R(\theta)\Vert\leq c\}$
exists.

There are many other mixing-time bounds for different algorithms,
and different types of models \citep{Levin2006Markovchainsand}. The
most common algorithms are univariate Gibbs sampling (often called
Glauber dynamics in the mixing time literature) and Swendsen-Wang
sampling. The Ising model and Potts models are the most common distributions
studied, either with a grid or fully-connected graph structure. Often,
the motivation for studying these systems is to understand physical
systems, or to mathematically characterize phase-transitions in mixing
time that occur as interactions strengths vary. As such, many existing
bounds assume uniform interaction strengths. For all these reasons,
these bounds typically require some adaptation for a learning setting.

\section{Main Results}

\subsection{Lipschitz Gradient}

For lack of space, detailed proofs are postponed to the appendix.
However, informal proof sketches are provided to give some intuition
for results that have longer proofs. Our first main result is that
the regularized log-likelihood has a Lipschitz gradient.
\begin{thm}
\label{thm:Lipschitz-constant}The regularized log-likelihood gradient
is $L$-Lipschitz with $L=4R_{2}^{2}+\lambda$, i.e. 
\[
\Vert f'(\theta)-f'(\phi)\Vert_{2}\leq(4R_{2}^{2}+\lambda)\Vert\theta-\phi\Vert_{2}.
\]
\end{thm}
\begin{proof}
[Proof sketch] It is easy, by the triangle inequality, that $\Vert f'(\theta)-f'(\phi)\Vert_{2}\leq\Vert\frac{dA}{d\theta}-\frac{dA}{d\phi}\Vert_{2}+\lambda\Vert\theta-\phi\Vert_{2}$.
Next, using the assumption that $\Vert t(x)\Vert_{2}\leq R_{2}$,
one can bound that $\Vert\frac{dA}{d\theta}-\frac{dA}{d\phi}\Vert_{2}\leq2R_{2}\Vert p_{\theta}-p_{\phi}\Vert_{TV}$.
Finally, some effort can bound that $\Vert p_{\theta}-p_{\phi}\Vert_{TV}\leq2R_{2}\Vert\theta-\phi\Vert_{2}.$
\end{proof}

\subsection{Convex convergence}

Now, our first major result is a guarantee on the convergence that
is true both in the regularized case where $\lambda>0$ and the unregularized
case where $\lambda=0.$ 
\begin{thm}
\label{thm:convergence-convex}With probability at least $1-\delta$,
at long as $M\geq3K/\log(\frac{1}{\delta})$, Algorithm 1 will satisfy
\[
f\left(\frac{1}{K}\sum_{k=1}^{K}\theta_{k}\right)-f(\theta^{*})\leq\frac{8R_{2}^{2}}{KL}\left(\frac{L\Vert\theta_{0}-\theta^{*}\Vert_{2}}{4R_{2}}+\log\frac{1}{\delta}+\frac{K}{\sqrt{M}}+KC\alpha^{v}\right)^{2}.
\]
\end{thm}
\begin{proof}
[Proof sketch]First, note that $f$ is convex, since the Hessian
of $f$ is the covariance of $t(X)$ when $\lambda=0$ and $\lambda>0$
only adds a quadratic. Now, define the quantity $d_{k}=\frac{1}{M}\sum_{m=1}^{M}t(X_{m}^{k})-\mathbb{E}_{q_{k}}[t(X)]$
to be the difference between the estimated expected value of $t(X)$
under $q_{k}$ and the true value. An elementary argument can bound
the expected value of $\Vert d_{k}\Vert$, while the Efron-Stein inequality
can bounds its variance. Using both of these bounds in Bernstein's
inequality can then show that, with probability $1-\delta$, $\sum_{k=1}^{K}\Vert d_{k}\Vert\leq2R_{2}(K/\sqrt{M}+\log\frac{1}{\delta}).$
Finally, we can observe that $\sum_{k=1}^{K}\Vert e_{k}\Vert\leq\sum_{k=1}^{K}\Vert d_{k}\Vert+\sum_{k=1}^{K}\Vert\mathbb{E}_{q_{k}}[t(X)]-\mathbb{E}_{p_{\theta_{k}}}[t(X)]\Vert_{2}$.
By the assumption on mixing speed, the last term is bounded by $2KR_{2}C\alpha^{v}$.
And so, with probability $1-\delta$, $\sum_{k=1}^{K}\Vert e_{k}\Vert\leq2R_{2}(K/\sqrt{M}+\log\frac{1}{\delta})+2KR_{2}C\alpha^{v}.$
Finally, a result due to Schmidt et al. \citep{Schmidt2011ConvergenceRatesInexact}
on the convergence of gradient descent with errors in estimated gradients
gives the result.
\end{proof}
Intuitively, this result has the right character. If $M$ grows on
the order of $K^{2}$ and $v$ grows on the order of $\log K/(-\log\alpha)$,
then all terms inside the quadratic will be held constant, and so
if we set $K$ of the order $1/\epsilon$, the sub-optimality will
on the order of $\epsilon$ with a total computational effort roughly
on the order of $(1/\epsilon)^{3}\log(1/\epsilon).$ The following
results pursue this more carefully. Firstly, one can observe that
a minimum amount of work must be performed.
\begin{thm}
For $a,b,c,\alpha>0,$ if $K,M,v>0$ are set so that $\frac{1}{K}(a+b\frac{K}{\sqrt{M}}+Kc\alpha^{v})^{2}\leq\epsilon$,
then
\[
KMv\geq\frac{a^{4}b^{2}}{\epsilon^{3}}\frac{\log\frac{ac}{\epsilon}}{(-\log\alpha)}.
\]
Since it must be true that $a/\sqrt{K}+b\sqrt{K/M}+\sqrt{K}c\alpha^{v}\leq\sqrt{\epsilon}$,
each of these three terms must also be at most $\sqrt{\epsilon},$
giving lower-bounds on $K$, $M$, and $v$. Multiplying these gives
the result. 
\end{thm}
Next, an explicit schedule for $K,$ $M$, and $v$ is possible, in
terms of a convex set of parameters $\beta_{1},\beta_{2},\beta_{3}$.
Comparing this to the lower-bound above shows that this is not too
far from optimal.
\begin{thm}
\label{thm:parameterized-convex-ub}Suppose that $a,b,c,\alpha>0.$
If $\beta_{1}+\beta_{2}+\beta_{3}=1$, $\beta_{1},\beta_{2},\beta_{3}>0$,
then setting $K=\frac{a^{2}}{\beta_{1}^{2}\epsilon},\,M=(\frac{ab}{\beta_{1}\beta_{2}\epsilon})^{2},\,v=\log\frac{ac}{\beta_{1}\beta_{3}\epsilon}/(-\log\alpha)$
is sufficient to guarantee that $\frac{1}{K}(a+b\frac{K}{\sqrt{M}}+Kc\alpha^{v})^{2}\leq\epsilon$
with a total work of\vspace{-5pt}
\[
KMv=\frac{1}{\beta_{1}^{4}\beta_{2}^{2}}\frac{a^{4}b^{2}}{\epsilon^{3}}\frac{\log\frac{ac}{\beta_{1}\beta_{3}\epsilon}}{(-\log\alpha)}.
\]
Simply verify that the $\epsilon$ bound holds, and multiply the terms
together.
\end{thm}
For example, setting $\beta_{1}=0.66$, $\beta_{2}=0.33$ and $\beta_{3}=0.01$
gives that $KMv\approx48.4\frac{a^{4}b^{2}}{\epsilon^{3}}\frac{\log\frac{ac}{\epsilon}+5.03}{(-\log\alpha)}.$
Finally, we can give an explicit schedule for $K$, $M$, and $v$,
and bound the total amount of work that needs to be performed.
\begin{thm}
If $D\geq\max\left(\Vert\theta_{0}-\theta^{*}\Vert_{2},\frac{4R_{2}}{L}\log\frac{1}{\delta}\right)$,
then for all $\epsilon$ there is a setting of $K,M,v$ such that
$f(\frac{1}{K}\sum_{k=1}^{K}\theta_{k})-f(\theta^{*})\leq\epsilon_{f}$
with probability $1-\delta$ and 
\begin{eqnarray*}
KMv & \leq & \frac{32LR_{2}^{2}D^{4}}{\beta_{1}^{4}\beta_{2}^{2}\epsilon_{f}^{3}(1-\alpha)}\log\frac{4DR_{2}C}{\beta_{1}\beta_{3}\epsilon_{f}}.
\end{eqnarray*}
[Proof sketch] This follows from setting $K$, $M$, and $v$ as
in Theorem \ref{thm:parameterized-convex-ub} with $a=L\Vert\theta_{0}-\theta^{*}\Vert_{2}/(4R_{2})+\log\frac{1}{\delta}$,
$b=1$, $c=C$, and $\epsilon=\epsilon_{f}L/(8R_{2}^{2})$.
\end{thm}

\subsection{Strongly Convex Convergence}

This section gives the main result for convergence that is true only
in the regularized case where $\lambda>0.$ Again, the main difficulty
in this proof is showing that the sum of the errors of estimated gradients
at each iteration is small. This is done by using a concentration
inequality to show that the error of each estimated gradient is small,
and then applying a union bound to show that the sum is small. The
main result is as follows.
\begin{thm}
When the regularization constant obeys $\lambda>0$, with probability
at least $1-\delta$ Algorithm 1 will satisfy
\[
\Vert\theta_{K}-\theta^{*}\Vert_{2}\leq(1-\frac{\lambda}{L})^{K}\Vert\theta_{0}-\theta^{*}\Vert_{2}+\frac{L}{\lambda}\left(\sqrt{\frac{R_{2}}{2M}}\left(1+\sqrt{2\log\frac{K}{\delta}}\right)+2R_{2}C\alpha^{v}\right).
\]
\end{thm}
\begin{proof}
[Proof sketch]When $\lambda=0$, $f$ is convex (as in Theorem \ref{thm:convergence-convex})
and so is strongly convex when $\lambda>0$. The basic proof technique
here is to decompose the error in a particular step as $\Vert e_{k+1}\Vert_{2}\leq\Vert\frac{1}{M}\sum_{i=1}^{M}t(x_{i}^{k})-\mathbb{E}_{q_{k}}[t(X)]\Vert_{2}+\Vert\mathbb{E}_{q_{k}}[t(X)]-\mathbb{E}_{p_{\theta_{k}}}[t(X)]\Vert_{2}.$
A multidimensional variant of Hoeffding's inequality can bound the
first term, with probability $1-\delta'$ by $R_{2}(1+\sqrt{2\log\frac{1}{\delta}})/\sqrt{M}$,
while our assumption on mixing speed can bound the second term by
$2R_{2}C\alpha^{v}$. Applying this to all iterations using $\delta'=\delta/K$
gives that all errors are simultaneously bounded as before. This can
then be used in another result due to Schmidt et al. \citep{Schmidt2011ConvergenceRatesInexact}
on the convergence of gradient descent with errors in estimated gradients
in the strongly convex case.
\end{proof}
A similar proof strategy could be used for the convex case where,
rather than directly bounding the sum of the norm of errors of all
steps using the Efron-Stein inequality and Bernstein's bound, one
could simply bound the error of each step using a multidimensional
Hoeffding-type inequality, and then apply this with probability $\delta/K$
to each step. This yields a slightly weaker result than that shown
in Theorem \ref{thm:convergence-convex}. The reason for applying
a uniform bound on the errors in gradients here is that Schmidt et
al.'s bound \citep{Schmidt2011ConvergenceRatesInexact} on the convergence
of proximal gradient descent on strongly convex functions depends
not just on the sum of the norms of gradient errors, but a non-uniform
weighted variant of these.

Again, we consider how to set parameters to guarantee that $\theta_{K}$
is not too far from $\theta^{*}$ with a minimum amount of work. Firstly,
we show a lower-bound.
\begin{thm}
\label{thm:strongly-convex-lb}Suppose $a,b,c>0$. Then for any $K,M,v$
such that $\gamma^{K}a+\frac{b}{\sqrt{M}}\sqrt{\log(K/\delta)}+c\alpha^{v}\leq\epsilon.$
it must be the case that\vspace{-5pt}
\[
KMv\geq\frac{b^{2}}{\epsilon^{2}}\frac{\log\frac{a}{\epsilon}\log\frac{c}{\epsilon}}{(-\log\gamma)(-\log\alpha)}\log\left(\frac{\log\frac{a}{\epsilon}}{\delta(-\log\gamma)}\right).
\]
[Proof sketch] This is established by noticing that $\gamma^{K}a$,
$\frac{b}{\sqrt{M}}\sqrt{\log\frac{K}{\delta}}$, and $c\alpha^{v}$
must each be less than $\epsilon$, giving lower bounds on $K$, $M$,
and $v$.
\end{thm}
Next, we can give an explicit schedule that is not too far off from
this lower-bound.
\begin{thm}
Suppose that $a,b,c,\alpha>0.$ If $\beta_{1}+\beta_{2}+\beta_{3}=1$,
$\beta_{i}>0$, then setting $K=\log(\frac{a}{\beta_{1}\epsilon})/(-\log\gamma),\,M=\frac{b^{2}}{\epsilon^{2}\beta_{2}^{2}}\left(1+\sqrt{2\log(K/\delta)}\right)^{2}$
and $v=\log\left(\frac{c}{\beta_{3}\epsilon}\right)/(-\log\alpha)$
is sufficient to guarantee that $\gamma^{K}a+\frac{b}{\sqrt{M}}(1+\sqrt{2\log(K/\delta)})+c\alpha^{v}\leq\epsilon$
with a total work of at most 
\[
KMV\leq\frac{b^{2}}{\epsilon^{2}\beta_{2}^{2}}\frac{\log\left(\frac{a}{\beta_{1}\epsilon}\right)\log\left(\frac{c}{\beta_{3}\epsilon}\right)}{(-\log\gamma)(-\log\alpha)}\left(1+\sqrt{2\log\frac{\log(\frac{a}{\beta_{1}\epsilon})}{\delta(-\log\gamma)}}\right)^{2}.
\]
\vspace{-15pt}
\end{thm}
For example, if you choose $\beta_{2}=1/\sqrt{2}$ and $\beta_{1}=\beta_{3}=(1-1/\sqrt{2})/2\approx0.1464$,
then this varies from the lower-bound in Theorem \ref{thm:strongly-convex-lb}
by a factor of two, and a multiplicative factor of $1/\beta_{3}\approx6.84$
inside the logarithmic terms. 
\begin{cor}
\label{cor:explicit-construction-strongly-convex}If we choose $K\geq\frac{L}{\lambda}\log\left(\frac{\Vert\theta_{0}-\theta\Vert_{2}}{\beta_{1}\epsilon}\right),$
$M\geq\frac{L^{2}R_{2}}{2\epsilon^{2}\beta_{2}^{2}\lambda^{2}}\left(1+\sqrt{2\log(K/\delta)}\right)^{2},$
and $v\geq\frac{1}{1-\alpha}\log\left(2LR_{2}C/(\beta_{3}\epsilon\lambda)\right)$,
then $\Vert\theta_{K}-\theta^{*}\Vert_{2}\leq\epsilon_{\theta}$ with
probability at least $1-\delta$, and the total amount of work is
bounded by\vspace{-5pt}
\[
KMv\leq\frac{L^{3}R_{2}}{2\epsilon_{\theta}^{2}\beta_{2}^{2}\lambda^{3}(1-\alpha)}\log\left(\frac{\Vert\theta_{0}-\theta\Vert_{2}}{\beta_{1}\epsilon_{\theta}}\right)\left(1+\sqrt{2\log\left(\frac{L}{\lambda\delta}\log\left(\frac{\Vert\theta_{0}-\theta\Vert_{2}}{\beta_{1}\epsilon_{\theta}}\right)\right)}\right)^{2}.
\]
\vspace{-20pt}
\end{cor}

\section{Discussion}

\vspace{-5pt}

An important detail in the previous results is that the convex analysis
gives convergence in terms of the regularized log-likelihood, while
the strongly-convex analysis gives convergence in terms of the parameter
distance. If we drop logarithmic factors, the amount of work necessary
for $\epsilon_{f}$ - optimality in the log-likelihood using the convex
algorithm is of the order $1/\epsilon_{f}^{3}$, while the amount
of work necessary for $\epsilon_{\theta}$ - optimality using the
strongly convex analysis is of the order $1/\epsilon_{\theta}^{2}$.
Though these quantities are not directly comparable, the standard
bounds on sub-optimality for $\lambda$-strongly convex functions
with $L$-Lipschitz gradients are that $\lambda\epsilon_{\theta}^{2}/2\leq\epsilon_{f}\leq L\epsilon_{\theta}^{2}/2.$
Thus, roughly speaking, when regularized for the strongly-convex analysis
shows that $\epsilon_{f}$ optimality in the log-likelihood can be
achieved with an amount of work only linear in $1/\epsilon_{f}$.

\vspace{-15pt}

\section{Example}

\begin{figure}
\includegraphics[scale=0.24]{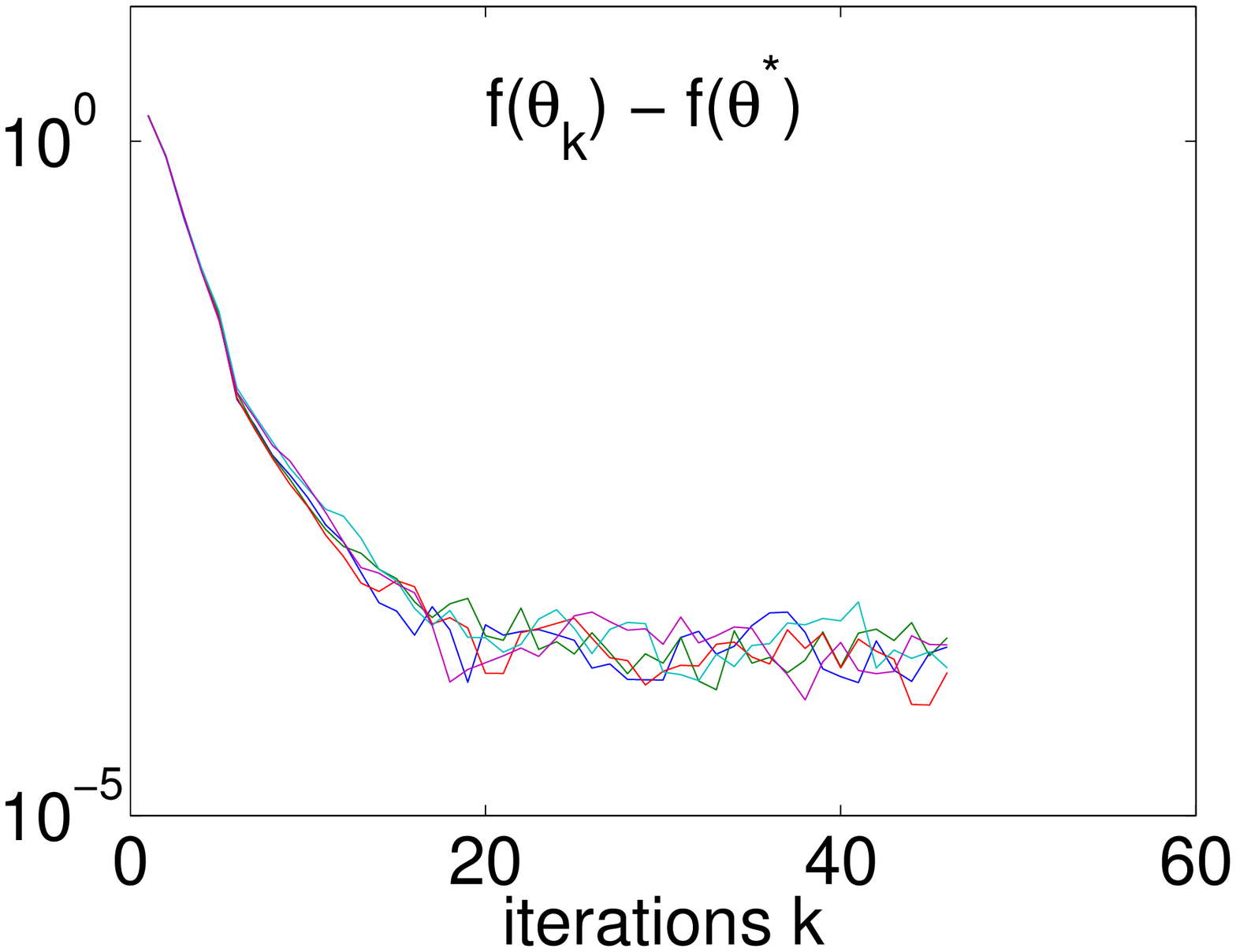}\includegraphics[scale=0.24]{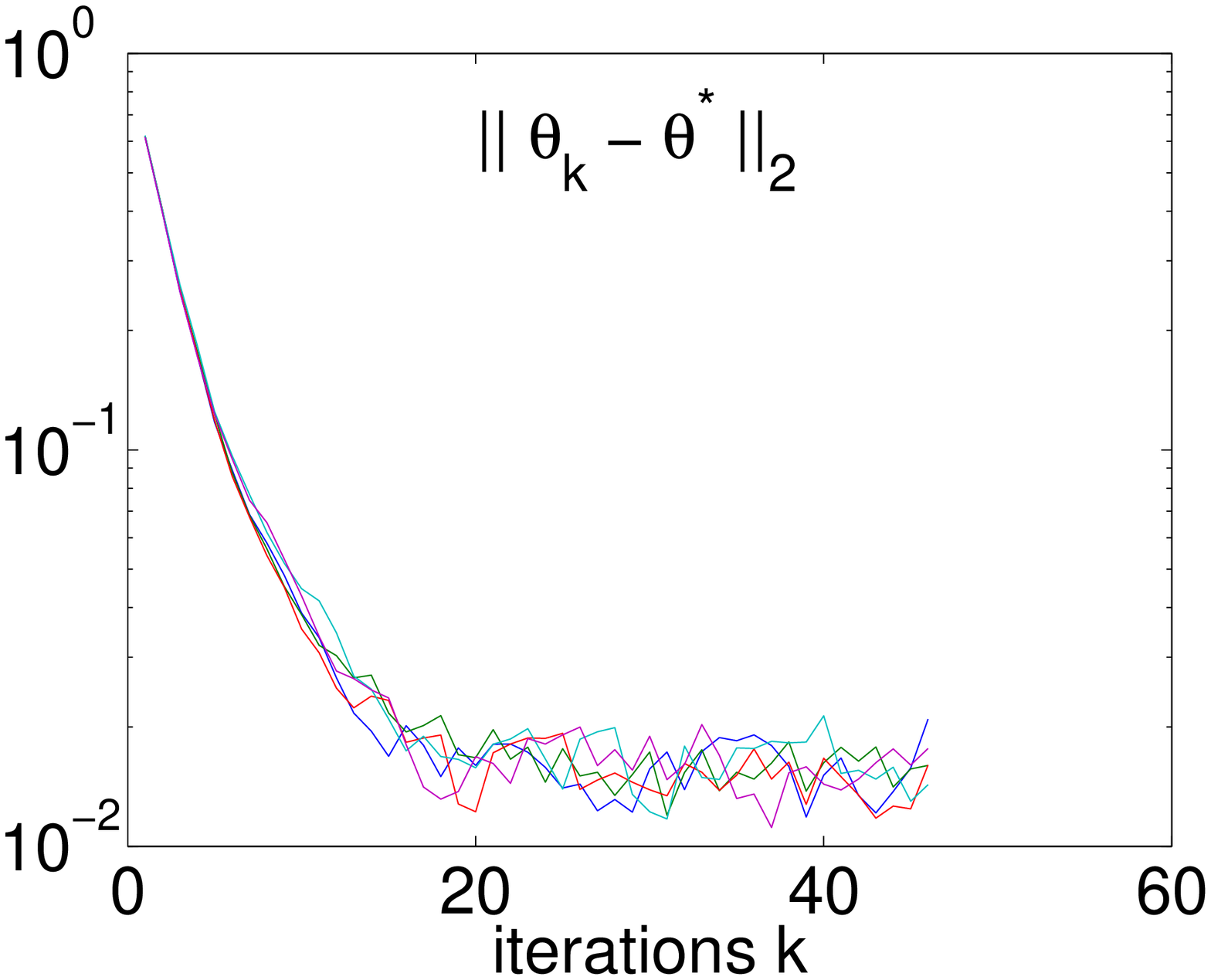}\includegraphics[bb=16bp 210bp 550bp 630bp,clip,scale=0.24]{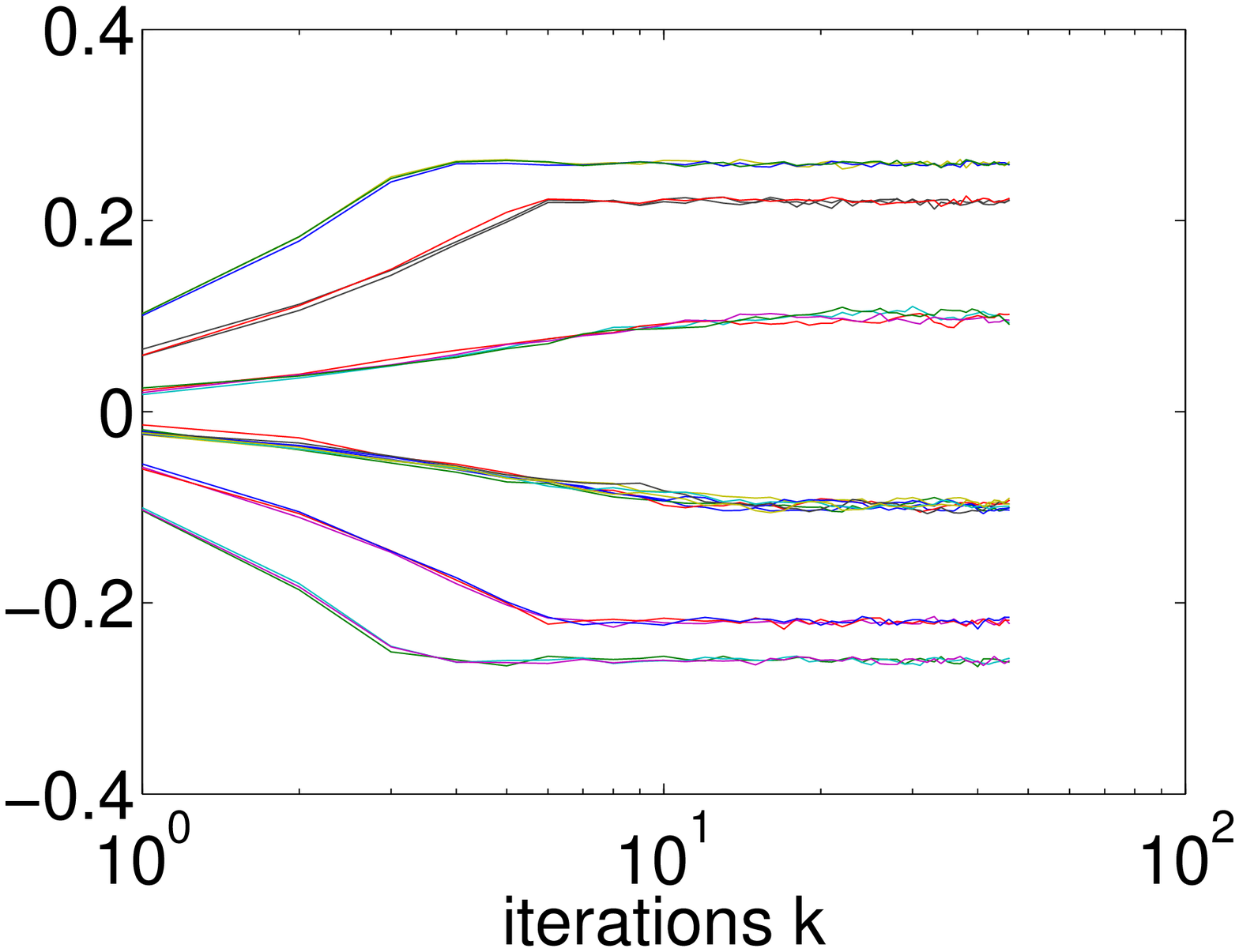}\includegraphics[bb=280bp 205bp 360bp 630bp,clip,scale=0.24]{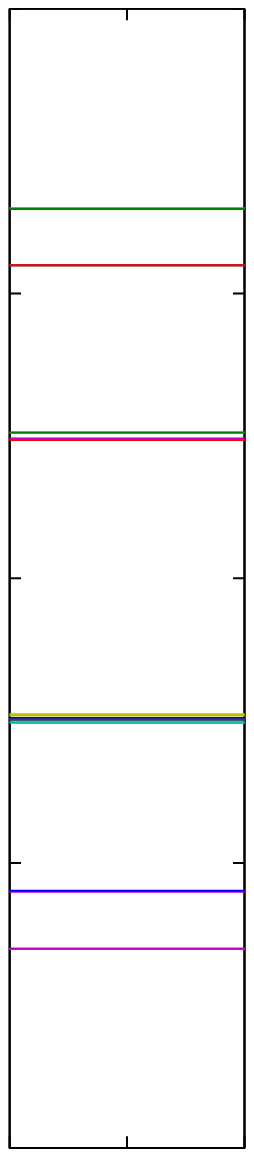}

\vspace{-5pt}\caption{Ising Model Example. Left: The difference of the current test log-likelihood
from the optimal log-likelihood on 5 random runs. Center: The distance
of the current estimated parameters from the optimal parameters on
5 random runs. Right: The current estimated parameters on one run,
as compared to the optimal parameters (far right).\label{fig:Ising-Model-Example}}
\end{figure}
While this paper claims no significant practical contribution,
it is useful to visualize an example. Take an Ising model $p(x)\propto\exp(\sum_{(i,j)\in\text{Pairs}}\theta_{ij}x_{i}x_{j})$
for $x_{i}\in\{-1,1\}$ on a $4\times4$ grid with 5 random vectors
as training data. The sufficient statistics are $t(x)=\{x_{i}x_{j}\vert(i,j)\in\text{Pairs}\}$,
and with 24 pairs, $\Vert t(x)\Vert_{2}\leq R_{2}=\sqrt{24}$. For
a fast-mixing set, constrain $\vert\theta_{ij}\vert\leq.2$ for all
pairs. Since the maximum degree is 4, $\tau(\epsilon)\leq\lceil\frac{N\log(N/\epsilon)}{1-4\tanh(.2)}\rceil$
. Fix $\lambda=1,$ $\epsilon_{\theta}=2$ and $\delta=0.1$. Though
the theory above suggests the Lipschitz constant $L=4R_{2}^{2}+\lambda=97$,
a lower value of $L=10$ is used, which converged faster in practice
(with exact or approximate gradients). Now, one can derive that $\Vert\theta_{0}-\theta^{*}\Vert_{2}\leq D=\sqrt{24\times(2\times.2)^{2}}$,
$C=\log(16)$ and $\alpha=\exp(-(1-4\tanh.2)/16)$. Applying Corollary
\ref{cor:explicit-construction-strongly-convex} with $\beta_{1}=.01$,
$\beta_{2}=.9$ and $\beta_{3}=.1$ gives $K=46$, $M=1533$ and $v=561$.
Fig. \ref{fig:Ising-Model-Example} shows the results. In practice,
the algorithm finds a solution tighter than the specified $\epsilon_{\theta}$,
indicating a degree of conservatism in the theoretical bound.

\vspace{-8pt}

\section{Conclusions}

\vspace{-8pt}

This section discusses some weaknesses of the above analysis, and
possible directions for future work. Analyzing complexity in terms
of the total sampling effort ignores the complexity of projection
itself. Since projection only needs to be done $K$ times, this time
will often be very small in comparison to sampling time. (This is
certainly true in the above example.) However, this might not be the
case if the projection algorithm scales super-linearly in the size
of the model.

Another issue to consider is how the samples are initialized. As far
as the proof of correctness goes, the initial distribution $r$ is
arbitrary. In the above example, a simple uniform distribution was
used. However, one might use the empirical distribution of the training
data, which is equivalent to contrastive divergence \citep{nan2005ContrastiveDivergenceLearning}.
It is reasonable to think that this will tend to reduce the mixing
time when the $p_{\theta}$ is close to the model generating the data.
However, the number of Markov chain transitions $v$ prescribed above
is larger than typically used with contrastive divergence, and Algorithm
\ref{fig:outline-and-cartoon} does not reduce the step size over
time. While it is common to regularize to encourage fast mixing with
contrastive divergence \citep[Section 10]{Hinton2010PracticalGuideto},
this is typically done with simple heuristic penalties. Further, contrastive
divergence is often used with hidden variables. Still, this provides
a bound for how closely a variant of contrastive divergence could
approximate the maximum likelihood solution.

The above analysis does not encompass the common strategy for maximum
likelihood learning where one maintains a ``pool'' of samples between
iterations, and initializes one Markov chain at each iteration from
each element of the pool. The idea is that if the samples at the previous
iteration were close to $p_{k-1}$ and $p_{k-1}$ is close to $p_{k}$,
then this provides an initialization close to the current solution.
However, the proof technique used here is based on the assumption
that the samples $x_{i}^{k}$ at each iteration are independent, and
so cannot be applied to this strategy.

\vspace{-5pt}

\subsubsection*{Acknowledgements}

\vspace{-6pt}

Thanks to Ivona Bez\'{a}kov\'{a}, Aaron Defazio, Nishant Mehta,
Aditya Menon, Cheng Soon Ong and Christfried Webers. NICTA is funded
by the Australian Government through the Dept. of Communications and
the Australian Research Council through the ICT Centre of Excellence
Program.

\clearpage{}

\let\oldbibliography\thebibliography
\renewcommand{\thebibliography}[1]{\oldbibliography{#1}
\setlength{\itemsep}{1.05pt}}

{\small{}\bibliographystyle{myplainnat}
\bibliography{/Users/jdomke/Dropbox/Papers/Bibliography/justindomke}
}{\small \par}

\clearpage{}

\section*{Appendix}

\section{Background}

\subsection{Optimization}

The main results in this paper rely strongly on the work of Schmidt
et al. \citep{Schmidt2011ConvergenceRatesInexact} on the convergence
of proximal gradient methods with errors in estimated gradients. The
first result used is the following theorem for the convergence of
gradient descent on convex functions with errors in the estimated
gradients.
\begin{thm}
\label{thm:opt-result-convex}(Special case of \citep[Proposition 1]{Schmidt2011ConvergenceRatesInexact})
Suppose that a function $f$ is convex with an $L$-Lipshitz gradient
(meaning $\Vert f'(\phi)-f'(\theta)\Vert_{2}\leq L\Vert\phi-\theta\Vert_{2})$.
If $\Theta$ is a closed convex set and one iterates
\[
\theta_{k}\leftarrow\Pi_{\Theta}\left[\theta_{k-1}-\frac{1}{L}\left(f'(\theta_{k-1})+e_{k}\right)\right],
\]
then, defining $\theta^{*}\in\arg\min_{\theta\in\Theta}f(\theta),$
for all $K\geq1$, we have, for $A_{K}:=\sum_{k=1}^{K}\frac{\Vert e_{k}\Vert}{L},$
that 
\begin{eqnarray*}
f\left(\frac{1}{K}\sum_{k=1}^{K}\theta_{k}\right)-f(\theta^{*}) & \leq & \frac{L}{2K}\left(\Vert\theta_{0}-\theta^{*}\Vert_{2}+2A_{K}\right)^{2}.
\end{eqnarray*}

\end{thm}
This section will show that this is indeed a special case of .\citep{Schmidt2011ConvergenceRatesInexact}
To start with, we simply restate exactly the previous result \citep[Proposition 1]{Schmidt2011ConvergenceRatesInexact},
with only trivial changes in notation.\begin{thm}
Assume that:
\begin{itemize}
\item $f$ is convex and has $L$-Lipschitz continuous gradient
\item $h$ is a lower semi-continuous proper convex function.
\item The function $r=f+h$ attains it's minimum at a certain $\theta^{*}\in\mathbb{R}^{n}$.
\item $\theta_{k}$ is an $\epsilon_{k}$-optimal solution, i.e. that
\[
\frac{L}{2}\Vert\theta_{k}-y\Vert^{2}+h(\theta_{k})\leq\epsilon_{k}+\min_{\theta\in\mathbb{R}^{n}}\frac{L}{2}\Vert\theta-y\Vert^{2}+h(\theta)
\]
where
\[
y=\theta_{k-1}-\frac{1}{L}\left(f'(\theta_{k-1})+e_{k}\right).
\]

\end{itemize}
Then, for all $K\geq1$, one has that
\[
r\left(\frac{1}{K}\sum_{k=1}^{K}\theta_{k}\right)-r(\theta^{*})\leq\frac{L}{2K}\left(\Vert\theta_{0}-\theta^{*}\Vert+2A_{K}+\sqrt{2B_{K}}\right)^{2}
\]
with
\[
A_{K}=\sum_{k=1}^{K}\left(\frac{\Vert e_{k}\Vert}{L}+\sqrt{\frac{2\epsilon_{k}}{L}}\right),~~~~B_{K}=\sum_{k=1}^{K}\frac{\epsilon_{k}}{K}.
\]

\end{thm}
The first theorem follows from this one by setting $h$ to be the
indicator function for the set $\Theta$, i.e.

\[
h(\theta)=\begin{cases}
0 & \theta\in\Theta\\
\infty & \theta\not\in\Theta
\end{cases}
\]
and assuming that $\epsilon_{k}=0$. By the convexity of $\Theta$,
$h$ will be a lower semi-continuous proper convex function. Further,
from the fact that $\Theta$ is closed, $r$ will attain its minimum.
Now, we verify that this results in the theorem statement at the start
of this section. $\theta_{k}$ takes the form 
\begin{eqnarray*}
\theta_{k} & = & \arg\min_{\theta\in\mathbb{R}^{n}}\frac{L}{2}\Vert\theta-y\Vert^{2}+h(\theta)\\
 & = & \arg\min_{\theta\in\Theta}\Vert\theta-y\Vert\\
 & = & \arg\min_{\theta\in\Theta}\Vert\theta-\theta_{k-1}+\frac{1}{L}\left(f'(\theta_{k-1})+e_{k}\right)\Vert\\
 & = & \Pi_{\Theta}\left[\theta_{k-1}-\frac{1}{L}\left(f'(\theta_{k-1})+e_{k}\right)\right].
\end{eqnarray*}

We will also use the following result for strongly-convex optimzation.
The special case follows from the same construction used above.

Next, consider the following result on optimization of strongly convex
functions, which follows from \citep{Schmidt2011ConvergenceRatesInexact}
by a very similar argument.
\begin{thm}
\label{thm:optimization-strongly-convex}(Special case of \citep[Proposition 3]{Schmidt2011ConvergenceRatesInexact})
Suppose that a function $f$ is $\lambda$-strongly convex with an
$L$-Lipshitz gradient (meaning $\Vert f'(\phi)-f'(\theta)\Vert_{2}\leq L\Vert\phi-\theta\Vert_{2})$.
If $\Theta$ is a closed convex set and one iterates
\[
\theta_{k}\leftarrow\Pi_{\Theta}\left[\theta_{k-1}-\frac{1}{L}\left(f'(\theta_{k-1})+e_{k}\right)\right],
\]
Then, defining $\theta^{*}=\arg\min_{\theta\in\Theta}f(\theta),$
for all $K\geq1$, we have, for $\bar{A}_{k}=\sum_{k=1}^{K}(1-\frac{\lambda}{L})^{-k}\frac{\Vert e_{k}\Vert}{L}$
that 
\begin{eqnarray*}
\Vert\theta_{K}-\theta^{*}\Vert_{2} & \leq & (1-\frac{\lambda}{L})^{K}\left(\Vert\theta_{0}-\theta^{*}\Vert_{2}+\bar{A}_{k}\right)
\end{eqnarray*}
\end{thm}
\begin{cor}
\label{cor:optimization-strongly-convex}Under the same conditions,
if $\Vert e_{k}\Vert\leq r$ for all $k$, then
\[
\Vert\theta_{K}-\theta^{*}\Vert_{2}\leq(1-\frac{\lambda}{L})^{K}\Vert\theta_{0}-\theta^{*}\Vert_{2}+\frac{rL}{\lambda}
\]
\end{cor}
\begin{proof}
Using the fact that $\sum_{k=1}^{K}a^{-k}=a^{-K}\sum_{k=0}^{K-1}a^{k}\leq a^{-K}\sum_{k=0}^{\infty}a^{k}=\frac{a^{-K}}{1-a}$,
we get that 
\[
\bar{A}_{K}\leq r\sum_{k=1}^{K}(1-\frac{\lambda}{L})^{-k}\leq r\frac{L}{\lambda}(1-\frac{\lambda}{L})^{-K},
\]
 and therefore that
\begin{eqnarray*}
\Vert\theta_{K}-\theta^{*}\Vert_{2} & \leq & (1-\frac{\lambda}{L})^{K}\left(\Vert\theta_{0}-\theta^{*}\Vert_{2}+r\frac{L}{\lambda}(1-\frac{\lambda}{L})^{-K}\right).
\end{eqnarray*}

\end{proof}

\subsection{Concentration Results}

Three concentration inequalities, are stated here for reference. The
first is Bernstein's inequality.
\begin{thm}
(Bernstein's inequality) Suppose $Z_{1},...,Z_{K}$ are independent
with mean $0$, that $\vert Z_{k}\vert\leq c$ and that $\sigma_{i}^{2}=\mathbb{V}[Z_{i}]$.
Then, if we define $\sigma^{2}=\frac{1}{K}\sum_{k=1}^{K}\sigma_{k}^{2}$,
\[
\mathbb{P}\left[\frac{1}{K}\sum_{k=1}^{K}Z_{k}>\epsilon\right]\leq\exp\left(-\frac{K\epsilon^{2}}{2\sigma^{2}+2c\epsilon/3}\right).
\]

\end{thm}
The second is the following Hoeffding-type bound to control the difference
between the expected value of $t(X)$ and the estimated value using
$M$ samples.
\begin{thm}
\label{thm:Hoeffding-fixed-delta-form}If $X_{1},...,X_{M}$ are independent
variables with mean $\mu$, and $\Vert X_{i}-\mu\Vert\leq c$, then
for all $\epsilon\geq0$, with probability at least $1-\delta$,
\[
\Vert\bar{X}-\mu\Vert\leq\sqrt{\frac{c}{4M}}\left(1+\sqrt{2\log\frac{1}{\delta}}\right).
\]
\end{thm}
\begin{proof}
Boucheron et al. \citep[Ex. 6.3]{Boucheron2013ConcentrationInequalities:Nonasymptotic}
show that, under the same conditions as stated, for all $s\geq\sqrt{v}$,
\[
\mathbb{P}\left[\Vert\bar{X}-\mu\Vert>\frac{s}{M}\right]\leq\exp\left(-\frac{(s-\sqrt{v})^{2}}{2v}\right),
\]
where $v=\frac{cM}{4}.$ We will fix $\delta$, and solve for the
appropriate $s$. If we set $\delta=\exp(-\frac{(s-\sqrt{v})^{2}}{2v}),$
then we have that $s=\sqrt{2v\log\frac{1}{\delta}}+\sqrt{v},$ meaning
that, with probability at least $1-\delta$,
\begin{eqnarray*}
\Vert\bar{X}-\mu\Vert & \leq & \frac{1}{M}\left(\sqrt{2\frac{cM}{4}\log\frac{1}{\delta}}+\sqrt{\frac{cM}{4}}\right),
\end{eqnarray*}
which is equivalent to the result with a small amount of manipulation.
\end{proof}
The third is the Efron-Stein inequality \citep[Theorem 3.1]{Boucheron2013ConcentrationInequalities:Nonasymptotic}.
\begin{thm}
If $X=(X_{1},...,X_{m})$ is a vector of independent random variables
and $f(X)$ is a square-integrable function, then

\[
\mathbb{V}[f(X)]\leq\frac{1}{2}\sum_{i=1}^{M}\mathbb{E}\left[\left((f(X)-f(X^{(i)})\right)^{2}\right],
\]
where $X^{(i)}$ is $X$ with $X_{i}$ independently re-drawn, i.e.
\[
X^{(i)}=(X_{1},...,X_{i-1},X'_{i'},X_{i+1},...,X_{m}).
\]

\end{thm}

\section{Preliminary Results}

A result that we will use several times below is that, for $0<\alpha<1$,
$-\frac{1}{\log(\alpha)}\leq\frac{1}{1-\alpha}$. This bound is tight
in the limit that $\alpha\rightarrow1$.
\begin{lem}
\label{lem:expectation-error-vs-tv-error}The difference of two estimated
mean vectors is bounded by 
\[
\Vert\mathbb{E}_{q}[t(X)]-\mathbb{E}_{p}[t(X)]\Vert_{2}\leq2R_{2}\Vert q-p\Vert_{TV}.
\]
\end{lem}
\begin{proof}
Let the distribution functions of $p$ and $q$ be $P$ and $Q$,
respectively. Then, we have that
\begin{alignat*}{1}
\Vert\underset{p}{\mathbb{E}}[t(X)]-\underset{q}{\mathbb{E}}[t(X)]\Vert_{2} & =\left\Vert \int_{x}t(x)\left(dP(x)-dQ(x)\right)\right\Vert _{2}\\
 & \leq\int_{x}\left\vert dP(x)-dQ(x)\right\vert \cdot\Vert t(x)\Vert_{2}.
\end{alignat*}
Using the definition of total-variation distance, and the bound that
$\Vert t(x)\Vert_{2}\leq R_{2}$ gives the result.\end{proof}
\begin{lem}
\label{lem:partition-function-diff}If $1/a+1/b=1$, then the difference
of two log-partition functions is bounded by
\[
\vert A(\theta)-A(\phi)\vert\leq R_{a}\Vert\theta-\phi\Vert_{b}.
\]
\end{lem}
\begin{proof}
By the Lagrange remainder theorem, there must exist some $\gamma$
on the line segment between $\theta$ and $\phi$ such that $A(\phi)=A(\theta)+(\phi-\theta)^{T}\nabla_{\gamma}A(\gamma).$
Thus, applying Hölder's inequality, we have that
\[
\vert A(\phi)-A(\theta)\vert=\vert(\phi-\theta)^{T}\nabla_{\gamma}A(\gamma)\vert\leq\Vert\phi-\theta\Vert_{b}\cdot\Vert\nabla_{\gamma}A(\gamma)\Vert_{a}.
\]
The result follows from the fact that $\Vert\nabla_{\gamma}A(\gamma)\Vert_{a}=\Vert\mathbb{E}_{p_{\gamma}}t(X)\Vert_{a}\leq R_{a}.$
\end{proof}
Next, we observe that the total variation distance between $p_{\theta}$
and $p_{\phi}$ is bounded by the distance between $\theta$ and $\phi$.
\begin{thm}
\label{thm:tv_in_terms_of_parameter_diff}If $1/a+1/b=1$, then the
difference of distributions is bounded by
\[
\Vert p_{\theta}-p_{\phi}\Vert_{TV}\leq2R_{a}\Vert\theta-\phi\Vert_{b}.
\]
\end{thm}
\begin{proof}
If we assume that $p_{\theta}$ is a density, we can decompose the
total-variation distance as
\begin{alignat*}{1}
 & ||p_{\theta}-p_{\phi}||_{TV}\\
 & =\frac{1}{2}\int_{x}p_{\theta}(x)|1-\frac{p{}_{\phi}(x)}{p_{\theta}(x)}|\\
 & =\frac{1}{2}\int_{x}p_{\theta}(x)\left|1-\exp\left((\phi-\theta)\cdot t(x)-A(\phi)+A(\theta)\right)\right|\\
 & \leq\frac{1}{2}\int_{x}p_{\theta}(x)\left|1-\exp\left|(\phi-\theta)\cdot t(x)-A(\phi)+A(\theta)\right|\right|.
\end{alignat*}

If $p_{\theta}$ is a distribution, the analogous expression is true,
replacing the integral over $x$ with a sum.

We can upper-bound the quantity inside $\exp$ by applying Hölder's
inequality and the previous Lemma as
\begin{align*}
 & \vert(\phi-\theta)\cdot t(x)-A(\phi)+A(\theta))\vert\\
 & \leq\vert(\phi-\theta)\cdot t(x)\vert+\vert A(\phi)-A(\theta))\vert\\
 & \leq2R_{a}\Vert\theta-\phi\Vert_{b}.
\end{align*}
 From which we have that
\[
\Vert p_{\theta}-p_{\phi}\Vert_{TV}\leq\frac{1}{2}\left\vert 1-\exp\left(2R_{a}\Vert\theta-\phi\Vert_{b}\right)\right\vert .
\]

If $2R_{a}\Vert\theta-\phi\Vert_{b}>1$, the theorem is obviously
true, since $\Vert\cdot\Vert_{TV}\leq1$. Suppose instead that that
$2R_{a}\Vert\theta-\phi\Vert_{b}\leq1$. If $0\leq c\leq1$, then
$\frac{1}{2}\vert1-\exp(c)\vert\leq c\frac{e-1}{2}$. Applying this
with $c=2R_{a}\Vert\theta-\phi\Vert_{b}$ gives that $||p_{\theta}-p_{\phi}||_{TV}\leq(e-1)R_{2}||\theta-\phi||_{b}$.
The result follows from the fact that $2>(e-1).$
\end{proof}

\section{Lipschitz Continuity}

This section shows that the ridge-regularized empirical log-likelihood
does indeed have a Lipschitz continuous gradient.
\begin{thm}
\label{thm:Lipschitz-constant-1}The regularized log-likelihood function
is $L$-Lipschitz with $L=4R_{2}^{2}+\lambda$, i.e. 
\[
\Vert f'(\theta)-f'(\phi)\Vert_{2}\leq(4R_{2}^{2}+\lambda)\Vert\theta-\phi\Vert_{2}.
\]
\end{thm}
\begin{proof}
We start by the definition of the gradient, with
\begin{align*}
\Vert f'(\theta)-f'(\phi)\Vert_{2} & =\left\Vert \left(\frac{dA}{d\theta}-\bar{t}+\lambda\theta\right)-\left(\frac{dA}{d\phi}-\bar{t}+\lambda\phi\right)\right\Vert _{2}\\
 & =\Vert\frac{dA}{d\theta}-\frac{dA}{d\phi}+\lambda(\theta-\phi)\Vert_{2}.\\
 & \leq\Vert\frac{dA}{d\theta}-\frac{dA}{d\phi}\Vert_{2}+\lambda\Vert\theta-\phi\Vert_{2}.
\end{align*}
Now, looking at the first two terms, we can apply Lemma \ref{lem:expectation-error-vs-tv-error}
to get that
\begin{align*}
\left\Vert \frac{dA}{d\theta}-\frac{dA}{d\phi}\right\Vert _{2} & =\left\Vert \mathbb{E}_{p_{\theta}}[t(X)]-\mathbb{E}_{p_{\phi}}[t(X)]\right\Vert _{2}\\
 & \leq2R_{2}\Vert p_{\theta}-p_{\phi}\Vert_{TV}.
\end{align*}

Observing by Theorem \ref{thm:tv_in_terms_of_parameter_diff} that
$\Vert p_{\theta}-p_{\phi}\Vert_{TV}\leq2R_{2}\Vert\theta-\phi\Vert_{2}$
gives that
\[
\Vert f'(\theta)-f'(\phi)\Vert_{2}\leq4R_{2}^{2}\Vert\theta-\phi\Vert_{2}+\lambda\Vert\theta-\phi\Vert_{2}
\]

\end{proof}

\section{Convex Convergence}

This section gives the main result for convergence this is true both
in the regularized case where $\lambda>0$ and the unregularized case
where $\lambda=0$. The main difficulty in this proof is showing that
the sum of the norms of the errors of estimated gradients is small.
\begin{thm}
\label{thm:single-step-estimation-expectation}Assuming that $X_{1},...,X_{M}$
are independent and identically distributed with mean $\mu$ and that
$\Vert X_{m}\Vert_{2}\leq R_{2}$, then

\[
\mathbb{E}\left[\Vert\frac{1}{M}\sum_{m=1}^{M}X_{m}-\mu\Vert_{2}\right]\leq\frac{2R_{2}}{\sqrt{M}}
\]
\end{thm}
\begin{proof}
Using that $\mathbb{E}\left[Z^{2}\right]=\mathbb{V}\left[Z\right]+\mathbb{E}\left[Z\right]^{2}$and
the fact that the variance is non-negative (Or simply Jensen's inequality),
we have  
\begin{eqnarray*}
\mathbb{E}\left[\Vert\frac{1}{M}\sum_{m=1}^{M}X_{m}-\mu\Vert_{2}\right]^{2} & \leq & \mathbb{E}\left[\Vert\frac{1}{M}\sum_{m=1}^{M}X_{m}-\mu\Vert_{2}^{2}\right]\\
 & = & \frac{1}{M}\mathbb{E}\left[\Vert X_{m}-\mu\Vert_{2}^{2}\right]\\
 & \le & \frac{1}{M}(2R_{2})^{2}\\
 & = & \frac{4R_{2}^{2}}{M}.
\end{eqnarray*}
Taking the square-root gives the result.\end{proof}
\begin{thm}
\label{thm:single-step-estimation-variance}Assuming that $X_{1},...,X_{M}$
are iid with mean $\mu$ and that $\Vert X_{m}\Vert\leq R_{2}$, then
\[
\mathbb{V}\left[\Vert\frac{1}{M}\sum_{m=1}^{M}X_{m}-\mu\Vert\right]\leq\frac{2R_{2}^{2}}{M}.
\]
\end{thm}
\begin{proof}
\begin{eqnarray*}
\mathbb{V}\left[\Vert\frac{1}{M}\sum_{m=1}^{M}X_{m}-\mu\Vert\right] & = & \mathbb{V}\left[\Vert\frac{1}{M}\sum_{m=1}^{M}(X_{m}-\mu)\Vert\right]\\
 & = & \frac{1}{M^{2}}\mathbb{V}\left[\Vert\sum_{m=1}^{M}(X_{m}-\mu)\Vert\right]
\end{eqnarray*}
Now, the Efron-Stein inequality tells us that
\[
\mathbb{V}[f(X_{1},...,X_{m})]\leq\frac{1}{2}\sum_{m'=1}^{M}\mathbb{E}\left[\left((f(X)-f(X^{(m')})\right)^{2}\right]
\]
where $X^{(m')}$ is $X$ with $X_{m'}$ independently re-drawn. Now,
we identify $f(X_{1},...,X_{m})=\Vert\sum_{m=1}^{M}(X_{m}-\mu)\Vert$
to obtain that
\[
\mathbb{V}\left[\Vert\sum_{m=1}^{M}(X_{m}-\mu)\Vert\right]\leq\frac{1}{2}\sum_{m'=1}^{M}\mathbb{E}\left[\left(\Vert\sum_{m=1}^{M}(X_{m}-\mu)\Vert-\Vert\sum_{m=1}^{M}(X_{m}^{(m')}-\mu)\Vert\right)^{2}\right].
\]
 Further, since we know that
\[
\sum_{m=1}^{M}(X_{m}^{(m')}-\mu)=\sum_{m=1}^{M}(X_{m}-\mu)+X_{m'}^{(m')}-X_{m'},
\]
we can apply that that $(\Vert a+b\Vert-\Vert a\Vert)^{2}\leq\Vert b\Vert^{2}$
to obtain that
\[
\left(\Vert\sum_{m=1}^{M}(X_{m}-\mu)\Vert-\Vert\sum_{m=1}^{M}(X_{m}^{(m')}-\mu)\Vert\right)^{2}=\Vert X_{m'}^{(m')}-X_{m'}\Vert^{2},
\]
 and so
\[
\mathbb{V}\left[\Vert\sum_{m=1}^{M}(X_{m}-\mu)\Vert\right]\leq\frac{1}{2}\sum_{m'=1}^{M}\mathbb{E}\left[\Vert X_{m'}^{(m')}-X_{m'}\Vert^{2}\right].
\]
And, since we assume that $\Vert X_{m}\Vert\leq R_{2},$ $\Vert X_{m'}^{(m')}-X_{m'}\Vert\leq2R_{2}$,
which leads to
\[
\mathbb{V}\left[\Vert\sum_{m=1}^{M}(X_{m}-\mu)\Vert\right]\leq2MR_{2}^{2},
\]
from which it follows that
\[
\mathbb{V}\left[\Vert\frac{1}{M}\sum_{m=1}^{M}X_{m}-\mu\Vert\right]\leq\frac{2R_{2}^{2}}{M}.
\]
\end{proof}
\begin{thm}
\label{thm:sum_of_gradient_error_convex}With probability at least
$1-\delta$,
\[
\sum_{k=1}^{K}\Vert\frac{1}{M}\sum_{i=1}^{M}t(x_{i}^{k})-\mathbb{E}_{q_{k}}[t(X)]\Vert_{2}\leq K\epsilon(\delta)+\frac{2R_{2}K}{\sqrt{M}},
\]
where $\epsilon(\delta)$ is the solution to
\begin{equation}
\delta=\exp\left(-\frac{K\epsilon^{2}}{4R_{2}^{2}/M+4R_{2}\epsilon/3}\right).\label{eq:delta_convex}
\end{equation}
\end{thm}
\begin{proof}
Let $d_{k}=\frac{1}{M}\sum_{i=1}^{M}t(x_{i}^{k})-\mathbb{E}_{q_{k}}[t(X)]$.
Applying Bernstein's inequality immediately gives us that 
\[
\mathbb{P}\left[\frac{1}{K}\sum_{k=1}^{K}\left(\Vert d_{k}\Vert_{2}-\mathbb{E}\Vert d_{k}\Vert_{2}\right)>\epsilon\right]\leq\exp\left(-\frac{K\epsilon^{2}}{2\sigma^{2}+2c\epsilon/3}\right).
\]
Here, we can bound $\sigma^{2}$ by
\[
\sigma^{2}=\frac{1}{K}\sum_{k=1}^{K}\sigma_{k}^{2}=\frac{1}{K}\sum_{k=1}^{K}\mathbb{V}\left[\Vert d_{k}\Vert_{2}-\mathbb{E}\Vert d_{k}\Vert_{2}\right]=\frac{1}{K}\sum_{k=1}^{K}\mathbb{V}\left[\Vert d_{k}\Vert_{2}\right]\leq\frac{2R_{2}^{2}}{M},
\]
where the final inequality follows from Theorem \ref{thm:single-step-estimation-variance}.
We also know that $\Vert d_{k}\Vert\leq2R_{2}=c$, from which we get
that
\begin{eqnarray*}
\mathbb{P}\left[\frac{1}{K}\sum_{k=1}^{K}\Vert d_{k}\Vert_{2}-\mathbb{E}[\Vert d_{k}\Vert_{2}]>\epsilon\right] & \leq & \exp\left(-\frac{K\epsilon^{2}}{4R_{2}^{2}/M+4R_{2}\epsilon/3}\right).
\end{eqnarray*}
So we have that, with probability $1-\delta$
\begin{eqnarray*}
\frac{1}{K}\sum_{k=1}^{K}\Vert d_{k}\Vert_{2}-\mathbb{E}[\Vert d_{k}\Vert_{2}] & \leq & \epsilon(\delta)\\
\frac{1}{K}\sum_{k=1}^{K}\Vert d_{k}\Vert_{2} & \leq & \epsilon(\delta)+\mathbb{E}[\Vert d_{k}\Vert_{2}]\\
 & \leq & \epsilon(\delta)+\frac{2R_{2}}{\sqrt{M}},
\end{eqnarray*}
where the final inequality follows from Theorem \ref{thm:single-step-estimation-expectation}.\end{proof}
\begin{cor}
If $M\geq3K/\log(\frac{1}{\delta})$, then with probability at least
$1-\delta$,
\[
\sum_{k=1}^{K}\Vert\frac{1}{M}\sum_{i=1}^{M}t(x_{i}^{k})-\mathbb{E}_{q_{k}}[t(X)]\Vert_{2}\leq2R_{2}\left(\frac{K}{\sqrt{M}}+\log\frac{1}{\delta}\right).
\]
\end{cor}
\begin{proof}
Solving Equation \ref{eq:delta_convex} for $\epsilon$ yields that
\[
\epsilon(\delta)=\frac{2R_{2}}{3K}\left(\log\frac{1}{\delta}+\sqrt{\left(\log\frac{1}{\delta}\right)^{2}+\frac{9K\log\frac{1}{\delta}}{M}}\right).
\]
Now, suppose that $\frac{3K}{M}\leq\log\frac{1}{\delta}$, as assumed
here. Then,
\begin{eqnarray*}
\epsilon(\delta) & \leq & \frac{2R_{2}}{3K}\left(\log\frac{1}{\delta}+\sqrt{\left(\log\frac{1}{\delta}\right)^{2}+3(\log\frac{1}{\delta})^{2}}\right)\\
 & \leq & \frac{2R_{2}}{3K}\left(\log\frac{1}{\delta}+2\log(\frac{1}{\delta})\right)\\
 & = & \frac{2R_{2}}{K}\log\frac{1}{\delta}.
\end{eqnarray*}

Substituting this bound into the result of Theorem \ref{thm:sum_of_gradient_error_convex}
gives the result.
\end{proof}
Now, we can prove the main result.
\begin{thm}
With probability at least $1-\delta$, at long as $M\geq3K/\log(\frac{1}{\delta})$,
\[
f\left(\frac{1}{K}\sum_{k=1}^{K}\theta_{k}\right)-f(\theta^{*})\leq\frac{8R_{2}^{2}}{KL}\left(\frac{L\Vert\theta_{0}-\theta^{*}\Vert_{2}}{4R_{2}}+\log\frac{1}{\delta}+\frac{K}{\sqrt{M}}+KC\alpha^{v}\right)^{2}.
\]
\end{thm}
\begin{proof}
Applying Theorem \ref{thm:opt-result-convex} gives that 
\begin{eqnarray*}
f\left(\frac{1}{K}\sum_{k=1}^{K}\theta_{k}\right)-f(\theta^{*}) & \leq & \frac{L}{2K}\left(\Vert\theta_{0}-\theta^{*}\Vert_{2}+2A_{K}\right)^{2},
\end{eqnarray*}
for $A_{K}=\frac{1}{L}\sum_{k=1}^{K}\Vert e_{k}\Vert,$ where
\begin{eqnarray*}
e_{k} & = & \frac{1}{M}\sum_{i=1}^{M}t(x_{i}^{k-1})-\bar{t}+\lambda\theta_{k-1}-f'(\theta_{k-1})\\
 & = & \frac{1}{M}\sum_{i=1}^{M}t(x_{i}^{k-1})-\mathbb{E}_{p_{k-1}}[t(X)].
\end{eqnarray*}

Now, we know that
\[
\sum_{k=1}^{K}\Vert e_{k}\Vert\leq\sum_{k=1}^{K}\Vert\frac{1}{M}\sum_{i=1}^{M}t(x_{i}^{k-1})-\mathbb{E}_{q_{k-1}}[t(X)]\Vert_{2}+\sum_{k=1}^{K}\Vert\mathbb{E}_{q_{k-1}}[t(X)]-\mathbb{E}_{p_{k-1}}[t(X)]\Vert_{2}.
\]
We have by Lemma \ref{lem:expectation-error-vs-tv-error} and the
assumption of mixing speed that 
\[
\Vert\mathbb{E}_{q_{k-1}}[t(X)]-\mathbb{E}_{p_{k-1}}[t(X)]\Vert_{2}\leq2R_{2}\Vert q_{k-1}-p_{k-1}\Vert_{TV}\leq2R_{2}C\alpha^{v}.
\]
Meanwhile, the previous Corollary tells us that, with probability
$1-\delta$,
\[
\sum_{k=1}^{K}\Vert\frac{1}{M}\sum_{i=1}^{M}t(x_{i}^{k-1})-\mathbb{E}_{q_{k-1}}[t(X)]\Vert_{2}\leq2R_{2}\left(\frac{K}{\sqrt{M}}+\log\frac{1}{\delta}\right)
\]
Thus, we have that 
\begin{eqnarray*}
f\left(\frac{1}{K}\sum_{k=1}^{K}\theta_{k}\right)-f(\theta^{*}) & \leq & \frac{L}{2K}\left(\Vert\theta_{0}-\theta^{*}\Vert_{2}+\frac{2}{L}\left(2R_{2}\left(\frac{K}{\sqrt{M}}+\log\frac{1}{\delta}\right)+2R_{2}KC\alpha^{v}\right)\right)^{2}\\
 & = & \frac{L}{2K}\left(\Vert\theta_{0}-\theta^{*}\Vert_{2}+\frac{4R_{2}}{L}\left(\frac{K}{\sqrt{M}}+\log\frac{1}{\delta}+KC\alpha^{v}\right)\right)^{2}\\
 & = & \frac{8R_{2}^{2}}{KL}\left(\frac{L\Vert\theta_{0}-\theta^{*}\Vert_{2}}{4R_{2}}+\log\frac{1}{\delta}+\frac{K}{\sqrt{M}}+KC\alpha^{v}\right)^{2}.
\end{eqnarray*}

\end{proof}
Now, what we really want to do is guarantee that $f\left(\frac{1}{K}\sum_{k=1}^{K}\theta_{k}\right)-f(\theta^{*})\leq\epsilon$,
while ensuring the the total work $MKv$ is not too large. Our analysis
will use the following theorem.
\begin{thm}
Suppose that $a,b,c,\alpha>0.$ If $\beta_{1}+\beta_{2}+\beta_{3}=1$,
$\beta_{1},\beta_{2},\beta_{3}>0$, then setting
\[
K=\frac{a^{2}}{\beta_{1}^{2}\epsilon},\,\,\,M=(\frac{ab}{\beta_{1}\beta_{2}\epsilon})^{2},\,\,v=\frac{\log\frac{ac}{\beta_{1}\beta_{3}\epsilon}}{(-\log\alpha)}
\]
is sufficient to guarantee that $\frac{1}{K}\left(a+b\frac{K}{\sqrt{M}}+Kc\alpha^{v}\right)^{2}\leq\epsilon$
with a total work of
\[
KMv=\frac{1}{\beta_{1}^{4}\beta_{2}^{2}}\frac{a^{4}b^{2}}{\epsilon^{3}}\frac{\log\frac{ac}{\beta_{1}\beta_{3}\epsilon}}{(-\log\alpha)}.
\]
\end{thm}
\begin{proof}
Firstly, we should verify the $\epsilon$ bound. We have that
\begin{eqnarray*}
a+b\frac{K}{\sqrt{M}}+Kc\alpha^{v} & = & a+b\frac{a^{2}}{\beta_{1}^{2}\epsilon}\frac{\beta_{1}\beta_{2}\epsilon}{ab}+\frac{a^{2}}{\beta_{1}^{2}\epsilon}c\frac{\beta_{1}\beta_{3}\epsilon}{ac}\\
 & = & a+a\frac{\beta_{2}}{\beta_{1}}+a\frac{\beta_{3}}{\beta_{1}},
\end{eqnarray*}
and hence that
\begin{eqnarray*}
\frac{1}{K}\left(a+b\frac{K}{\sqrt{M}}+Kc\alpha^{v}\right)^{2} & = & \frac{a^{2}}{K}\left(1+\frac{\beta_{2}}{\beta_{1}}+\frac{\beta_{3}}{\beta_{1}}\right)^{2}\\
 & = & \frac{1}{K}\frac{a^{2}}{\beta_{1}^{2}}\left(\beta_{1}+\beta_{2}+\beta_{3}\right)^{2}\\
 & \leq & \epsilon.
\end{eqnarray*}
Multiplying together th terms gives the second part of the result.
\end{proof}
We can also show that this solution is not too sub-optimal.
\begin{thm}
Suppose that $a,b,c,\alpha>0.$ If $K,M,v>0$ are set so that $\frac{1}{K}\left(a+b\frac{K}{\sqrt{M}}+Kc\alpha^{v}\right)^{2}\leq\epsilon$,
then
\[
KMv\geq\frac{a^{4}b^{2}}{\epsilon^{3}}\frac{\log\frac{ac}{\epsilon}}{(-\log\alpha)}.
\]
\end{thm}
\begin{proof}
The starting condition is equivalent to stating that
\[
\frac{a}{\sqrt{K}}+b\sqrt{\frac{K}{M}}+\sqrt{K}c\alpha^{v}\leq\sqrt{\epsilon}.
\]
Since all terms are positive, clearly each is less than $\sqrt{\epsilon}$.
From this follows that
\begin{eqnarray*}
K & \geq & \frac{a^{2}}{\epsilon}\\
M & \geq & \frac{b^{2}a^{2}}{\epsilon^{2}}\\
v & \geq & \frac{\log\frac{ac}{\epsilon}}{(-\log\alpha)}.
\end{eqnarray*}
Multiplying these together gives the result.\end{proof}
\begin{thm}
If $D\geq\max\left(\Vert\theta_{0}-\theta^{*}\Vert_{2},\frac{4R_{2}}{L}\log\frac{1}{\delta}\right)$,
then for all $\epsilon$ there is a setting of $KMv$ such that $f\left(\frac{1}{K}\sum_{k=1}^{K}\theta_{k}\right)-f(\theta^{*})\leq\epsilon_{f}$
with probability $1-\delta$ and 
\begin{eqnarray*}
KMv & \leq & \frac{32LR_{2}^{2}D^{4}}{\beta_{1}^{4}\beta_{2}^{2}\epsilon_{f}^{3}(1-\alpha)}\log\frac{4DR_{2}C}{\beta_{1}\beta_{3}\epsilon_{f}}\\
 & = & \mathcal{O}\left(\frac{LR_{2}^{2}D^{4}}{\epsilon_{f}^{3}(1-\alpha)}\log\frac{1}{\epsilon_{f}}\right)\\
 & = & \tilde{\mathcal{O}}\left(\frac{LR_{2}^{2}D^{4}}{\epsilon_{f}^{3}(1-\alpha)}\right).
\end{eqnarray*}
\end{thm}
\begin{proof}
So, we apply this to the original theorem. Our settings are

\[
f\left(\frac{1}{K}\sum_{k=1}^{K}\theta_{k}\right)-f(\theta^{*})\leq\frac{8R_{2}^{2}}{KL}\left(\frac{L\Vert\theta_{0}-\theta^{*}\Vert_{2}}{4R_{2}}+\log\frac{1}{\delta}+\frac{K}{\sqrt{M}}+KC\alpha^{v}\right)^{2}.
\]
\begin{eqnarray*}
a & = & \frac{L\Vert\theta_{0}-\theta^{*}\Vert_{2}}{4R_{2}}+\log\frac{1}{\delta}\\
b & = & 1\\
c & = & C\\
\epsilon & = & \frac{\epsilon_{f}L}{8R_{2}^{2}}
\end{eqnarray*}

Note that, by the definition of $D$, $a\leq\frac{LD}{2R_{2}}$ and
so $ac\leq\frac{LDC}{2R_{2}}$. Thus, the total amount of work is

\begin{eqnarray*}
KMv & = & \frac{1}{\beta_{1}^{4}\beta_{2}^{2}}\frac{a^{4}b^{2}}{\epsilon^{3}}\frac{\log\frac{\beta_{1}\beta_{3}\epsilon}{ac}}{\log\alpha}\\
 & = & \frac{1}{\beta_{1}^{4}\beta_{2}^{2}}\frac{a^{4}b^{2}}{\epsilon^{3}}\frac{\log\frac{ac}{\beta_{1}\beta_{3}\epsilon}}{-\log\alpha}\\
 & \leq & \frac{1}{\beta_{1}^{4}\beta_{2}^{2}}\frac{1}{\epsilon^{3}}\left(\frac{LD}{2R_{2}}\right)^{4}\frac{\log\frac{LDC}{\beta_{1}\beta_{3}2R_{2}\epsilon}}{\log\alpha}\\
 & = & \frac{1}{\beta_{1}^{4}\beta_{2}^{2}}\frac{8^{3}R_{2}^{6}}{\epsilon_{f}^{3}L^{3}}\left(\frac{LD}{2R_{2}}\right)^{4}\frac{\log\frac{LDC8R_{2}^{2}}{\beta_{1}\beta_{3}2R_{2}\epsilon_{f}L}}{\log\alpha}\\
 & = & \frac{1}{\beta_{1}^{4}\beta_{2}^{2}}\frac{32LD^{4}R_{2}^{2}}{\epsilon_{f}^{3}}\frac{\log\frac{4DR_{2}C}{\beta_{1}\beta_{3}\epsilon_{f}}}{\log\alpha}\\
 & \leq & \frac{32LD^{4}R_{2}^{2}}{\beta_{1}^{4}\beta_{2}^{2}\epsilon^{3}(1-\alpha)}\log\frac{4DR_{2}C}{\beta_{1}\beta_{3}\epsilon}.
\end{eqnarray*}

\end{proof}

\section{Strongly Convex Convergence}

This section gives the main result for convergence this is true both
only in the regularized case where $\lambda>0.$ Again, the main difficulty
in this proof is showing that the sum of the norms of the errors of
estimated gradients is small. This proof is relatively easier, as
we simply bound all errors to be small with high probability, rather
than jointly bounding the sum of errors.
\begin{lem}
With probability at least $1-\delta$,
\[
\Vert e_{k+1}\Vert_{2}\leq\frac{R_{2}}{\sqrt{M}}\left(1+\sqrt{2\log\frac{1}{\delta}}\right)+2R_{2}C\alpha^{v}
\]
\end{lem}
\begin{proof}
Once we have the difference of the distributions, we can go after
the error in the gradient estimate. By definition,
\begin{align*}
\Vert e_{k+1}\Vert_{2} & =\Vert\frac{1}{M}\sum_{i=1}^{M}t(x_{i}^{k})-\mathbb{E}_{p_{\theta_{k}}}[t(X)]\Vert_{2}\\
 & \leq\Vert\frac{1}{M}\sum_{i=1}^{M}t(x_{i}^{k})-\mathbb{E}_{q_{k}}[t(X)]\Vert_{2}\\
 & ~+\Vert\mathbb{E}_{q_{k}}[t(X)]-\mathbb{E}_{p_{\theta_{k}}}[t(X)]\Vert_{2}.
\end{align*}
 Consider the second term. We know by Lemma \ref{lem:expectation-error-vs-tv-error}
and the assumption of mixing speed 
\begin{equation}
\Vert\mathbb{E}_{q_{k}}[t(X)]-\mathbb{E}_{p_{k}}[t(X)]\Vert_{2}\leq2R_{2}\Vert q_{k}-p_{k}\Vert_{TV}\leq2R_{2}C\alpha^{v}.\label{eq:second-term-bound}
\end{equation}
Now, consider the first term. We know that $\mathbb{E}_{q_{k}}[t(X)]$
is the expected value of $\frac{1}{M}\sum_{i=1}^{M}t(x_{i}^{k})$.
We also know that $||t(x_{i}^{k})-\mathbb{E}_{q_{k}}[t(X)]||\leq2R_{2}.$
Thus, we can apply Theorem \ref{thm:Hoeffding-fixed-delta-form} to
get that, with probability $1-\delta$,
\begin{equation}
\left\Vert \frac{1}{M}\sum_{i=1}^{M}t\left(x_{i}^{k}\right)-\mathbb{E}_{q_{k}}[t(X)]\right\Vert \leq\frac{R_{2}}{\sqrt{M}}\left(1+\sqrt{2\log\frac{1}{\delta}}\right).\label{eq:first-term-bound}
\end{equation}

Adding together Equations \ref{eq:second-term-bound} and \ref{eq:first-term-bound}
gives the result.\end{proof}
\begin{thm}
With probability at least $1-\delta$,
\[
\Vert\theta_{K}-\theta^{*}\Vert_{2}\leq(1-\frac{\lambda}{L})^{K}\Vert\theta_{0}-\theta^{*}\Vert_{2}+\frac{L}{\lambda}\left(\sqrt{\frac{R_{2}}{2M}}\left(1+\sqrt{2\log\frac{K}{\delta}}\right)+2R_{2}C\alpha^{v}\right)
\]
\end{thm}
\begin{proof}
Apply the previous Lemma to bound bound on $\Vert e_{k+1}\Vert_{2}$
with probability at least $1-\delta'$ where $\delta'=\delta/K$.
Then, plug this into the main optimization result in Corollary \ref{cor:optimization-strongly-convex}.\end{proof}
\begin{thm}
Suppose $a,b,c>0$. Then for any $K,M,v$ such that $\gamma^{K}a+\frac{b}{\sqrt{M}}\sqrt{\log\frac{K}{\delta}}+c\alpha^{v}\leq\epsilon.$
it must be the case that
\[
KMv\geq\frac{b^{2}}{\epsilon^{2}}\frac{\log\frac{a}{\epsilon}\log\frac{c}{\epsilon}}{(-\log\gamma)(-\log\alpha)}\log\left(\frac{\log\frac{a}{\epsilon}}{\delta(-\log\gamma)}\right)
\]
\end{thm}
\begin{proof}
Clearly, we must have that each term is at most $\epsilon$, yielding
that
\begin{eqnarray*}
K & \geq & \frac{\log\frac{\epsilon}{a}}{\log\gamma}\\
M & \geq & \frac{b^{2}}{\epsilon^{2}}\log\frac{K}{\delta}\geq\frac{b^{2}}{\epsilon^{2}}\log\frac{\log\frac{\epsilon}{a}}{\delta\log\gamma}\\
v & \geq & \frac{\log(c/\epsilon)}{(-\log\alpha)}
\end{eqnarray*}
From this we obtain that
\begin{eqnarray*}
KMv & \geq & \frac{b^{2}}{\epsilon^{2}}\frac{\log\frac{a}{\epsilon}\log(c/\epsilon)}{(-\log\gamma)(-\log\alpha)}\log\left(\frac{\log\frac{a}{\epsilon}}{\delta(-\log\gamma)}\right).
\end{eqnarray*}
\end{proof}
\begin{thm}
Suppose that $a,b,c,\alpha>0.$ If $\beta_{1}+\beta_{2}+\beta_{3}=1$,
$\beta_{i}>0$, then setting
\begin{eqnarray*}
K & = & \log(\frac{a}{\beta_{1}\epsilon})/(-\log\gamma)\\
M & = & \frac{b^{2}}{\epsilon^{2}\beta_{2}^{2}}\left(1+\sqrt{2\log\frac{K}{\delta}}\right)^{2}\\
v & = & \log\left(\frac{c}{\beta_{3}\epsilon}\right)/(-\log\alpha)
\end{eqnarray*}
is sufficient to guarantee that $\gamma^{K}a+\frac{b}{\sqrt{M}}(1+\sqrt{2\log\frac{K}{\delta}})+c\alpha^{v}\leq\epsilon$
with a total work of at most 
\[
KMV\leq\frac{b^{2}}{\epsilon^{2}\beta_{2}^{2}}\frac{\log\left(\frac{a}{\beta_{1}\epsilon}\right)\log\left(\frac{c}{\beta_{3}\epsilon}\right)}{(-\log\gamma)(-\log\alpha)}\left(1+\sqrt{2\log\frac{\log(\frac{a}{\beta_{1}\epsilon})}{\delta(-\log\gamma)}}\right)^{2}
\]
\end{thm}
\begin{proof}
We define the errors so that
\begin{eqnarray*}
\gamma^{K}a & = & \epsilon\beta_{1}\\
\frac{b}{\sqrt{M}}(1+\sqrt{2\log\frac{K}{\delta}}) & = & \epsilon\beta_{2}\\
c\alpha^{v} & = & \epsilon\beta_{3}.
\end{eqnarray*}
Solving, we obtain that
\begin{eqnarray*}
K & = & \log(\frac{a}{\beta_{1}\epsilon})/(-\log\gamma)\\
M & = & \frac{b^{2}}{\epsilon^{2}\beta_{2}^{2}}\left(1+\sqrt{2\log\frac{K}{\delta}}\right)^{2}\\
v & = & \log\left(\frac{c}{\beta_{3}\epsilon}\right)/(-\log\alpha).
\end{eqnarray*}
This yields that the final amount of work is 
\begin{eqnarray*}
KMv & \leq & \frac{\log\left(\frac{a}{\beta_{1}\epsilon}\right)\log\left(\frac{c}{\beta_{3}\epsilon}\right)}{(-\log\gamma)(-\log\alpha)}\frac{b^{2}}{\epsilon^{2}\beta_{2}^{2}}\left(1+\sqrt{2\log\frac{\log(\frac{a}{\beta_{1}\epsilon})}{\delta(-\log\gamma)}}\right)^{2}
\end{eqnarray*}
\end{proof}
\begin{rem}
For example, you might choose $\beta_{2}=\frac{1}{2},\beta_{1}=\frac{1}{4}$
and $\beta_{3}=\frac{1}{4}$, in which case the total amount of work
is bounded by
\begin{eqnarray*}
KMv & \leq & \frac{4b^{2}}{\epsilon^{2}}\frac{\log\left(\frac{4a}{\epsilon}\right)\log\left(\frac{4c}{\epsilon}\right)}{(-\log\gamma)(-\log\alpha)}\left(1+\sqrt{2\log\frac{\log(\frac{4a}{\epsilon})}{\delta(-\log\gamma)}}\right)^{2}\\
 & = & \frac{4b^{2}}{\epsilon^{2}}\frac{\left(\log\left(\frac{a}{\epsilon}\right)+\log4\right)(\log\left(\frac{4c}{\epsilon}\right)+\log4)}{(-\log\gamma)(-\log\alpha)}\left(1+\sqrt{2\log\frac{\log(\frac{a}{\epsilon})+\log4}{\delta(-\log\gamma)}}\right)^{2}
\end{eqnarray*}

Or, if you choose $\beta_{2}=1/\sqrt{2}$ and $\beta_{1}=\beta_{3}=(1-1/\sqrt{2})/2\approx0.1464$,
then you get the bound of
\begin{eqnarray*}
KMV & \leq & \frac{2b^{2}}{\epsilon^{2}}\frac{(\log\left(\frac{a}{\epsilon}\right)+1.922)(\log\left(\frac{c}{\beta_{3}}\right)+1.922)}{(-\log\gamma)(-\log\alpha)}\left(1+\sqrt{2\log\frac{\log(\frac{a}{\epsilon})+1.922}{\delta(-\log\gamma)}}\right)^{2}
\end{eqnarray*}
which is not too much worse than the lower-bound.\end{rem}
\begin{cor}
If we choose \vspace{-10pt}

\begin{eqnarray*}
K & \geq & \frac{L}{\lambda}\log\left(\frac{\Vert\theta_{0}-\theta\Vert_{2}}{\beta_{1}\epsilon}\right)\\
M & \geq & \frac{L^{2}R_{2}}{2\epsilon^{2}\beta_{2}^{2}\lambda^{2}}\left(1+\sqrt{2\log\frac{K}{\delta}}\right)^{2}\\
v & \geq & \frac{1}{1-\alpha}\log\left(\frac{2LR_{2}C}{\beta_{3}\epsilon\lambda}\right)
\end{eqnarray*}
then $\Vert\theta_{K}-\theta^{*}\Vert_{2}\leq\epsilon$ with probability
at least $1-\delta$, and the total amount of work is bounded by
\[
KMv\leq\frac{1}{\epsilon^{2}}\left(\frac{L}{\lambda}\right)^{3}\frac{R_{2}}{2\beta_{2}^{2}(1-\alpha)}\log\left(\frac{\Vert\theta_{0}-\theta\Vert_{2}}{\beta_{1}\epsilon}\right)\left(1+\sqrt{2\log\left(\frac{L}{\lambda\delta}\log\left(\frac{\Vert\theta_{0}-\theta\Vert_{2}}{\beta_{1}\epsilon}\right)\right)}\right)^{2}
\]
\end{cor}
\begin{proof}
Apply the previous convergence theory to our setting. We equate
\[
(1-\frac{\lambda}{L})^{K}\Vert\theta_{0}-\theta^{*}\Vert_{2}+\frac{L}{\lambda}\left(\sqrt{\frac{R_{2}}{2M}}\left(1+\sqrt{2\log\frac{K}{\delta}}\right)+2R_{2}C\alpha^{v}\right)=\gamma^{K}a+\frac{b}{\sqrt{M}}(1+\sqrt{2\log\frac{K}{\delta}})+c\alpha^{v}.
\]
This requires the constants 
\begin{eqnarray*}
\gamma & = & 1-\frac{\lambda}{L}\\
a & = & \Vert\theta_{0}-\theta\Vert_{2}\\
b & = & \frac{L}{\lambda}\sqrt{\frac{R_{2}}{2}}\\
c & = & 2LR_{2}C/\lambda
\end{eqnarray*}
Thus, we will make the choices
\begin{eqnarray*}
K & = & \log(\frac{a}{\beta_{1}\epsilon})/(-\log\gamma)\\
 & \leq & \log(\frac{\Vert\theta_{0}-\theta\Vert_{2}}{\beta_{1}\epsilon})/(1-\gamma)\\
 & = & \frac{L}{\lambda}\log(\frac{\Vert\theta_{0}-\theta\Vert_{2}}{\beta_{1}\epsilon})\\
M & = & \frac{b^{2}}{\epsilon^{2}\beta_{2}^{2}}\left(1+\sqrt{2\log\frac{K}{\delta}}\right)^{2}\\
 & = & \frac{L^{2}R_{2}}{2\epsilon^{2}\beta_{2}^{2}\lambda^{2}}\left(1+\sqrt{2\log\frac{K}{\delta}}\right)^{2}\\
v & = & \log\left(\frac{c}{\beta_{3}\epsilon}\right)/(-\log\alpha)\\
 & = & \log\left(\frac{2LR_{2}C}{\beta_{3}\epsilon\lambda}\right)/(-\log\alpha)\\
 & \leq & \frac{1}{1-\alpha}\log\left(\frac{2LR_{2}C}{\beta_{3}\epsilon\lambda}\right)
\end{eqnarray*}
This means a total amount of work of
\begin{eqnarray*}
KMv & = & =\frac{L}{\lambda}\log(\frac{\Vert\theta_{0}-\theta\Vert_{2}}{\beta_{1}\epsilon})\frac{L^{2}R_{2}}{2\epsilon^{2}\beta_{2}^{2}\lambda^{2}(1-\alpha)}\left(1+\sqrt{2\log\left(\frac{L}{\lambda\delta}\log\left(\frac{\Vert\theta_{0}-\theta\Vert_{2}}{\beta_{1}\epsilon}\right)\right)}\right)^{2}\log\left(\frac{2LR_{2}C}{\beta_{3}\epsilon\lambda}\right)\\
 & = & \frac{1}{\epsilon^{2}}\left(\frac{L}{\lambda}\right)^{3}\frac{R_{2}}{2\beta_{2}^{2}(1-\alpha)}\log\left(\frac{\Vert\theta_{0}-\theta\Vert_{2}}{\beta_{1}\epsilon}\right)\left(1+\sqrt{2\log\left(\frac{L}{\lambda\delta}\log\left(\frac{\Vert\theta_{0}-\theta\Vert_{2}}{\beta_{1}\epsilon}\right)\right)}\right)^{2}.
\end{eqnarray*}
\end{proof}

\end{document}